\documentclass[letterpaper]{article} 

\usepackage{aaai2026}
\usepackage{times}  
\usepackage{helvet}  
\usepackage{courier}  
\usepackage[hyphens]{url}  
\usepackage{graphicx} 
\usepackage{natbib}  
\usepackage{caption} 
\usepackage{comment}

\usepackage{algorithm}
\usepackage{algorithmic}
\usepackage{newfloat}
\usepackage{listings}
\DeclareCaptionStyle{ruled}{labelfont=normalfont,labelsep=colon,strut=off} 
\floatstyle{ruled}
\newfloat{listing}{tb}{lst}{}
\floatname{listing}{Listing}
\usepackage{amsmath,amssymb,amsfonts,amsthm}  
\newtheorem{assumption}{Assumption}
\newtheorem{definition}{Definition}
\newtheorem{theorem}{Theorem}

\newtheorem{lemma}{Lemma}
\newtheorem{corollary}{Corollary}
\newtheorem{remark}{Remark}

\usepackage{booktabs,multirow,colortbl,xcolor} 
\usepackage{pifont} 
\usepackage{tikz}
\usetikzlibrary{arrows.meta,positioning,calc,decorations.pathreplacing}

\usepackage{pifont}           
\newcommand{\cmark}{\ding{51}}  
\newcommand{\xmark}{\ding{55}}  

\title{Constrained Online Convex Optimization with Memory and Predictions}

\author{
Mohammed Abdullah\textsuperscript{\rm 1,\rm 2},
George Iosifidis\textsuperscript{\rm 3},
Salah Eddine Elayoubi\textsuperscript{\rm 1},
Tijani Chahed\textsuperscript{\rm 2}
}
\affiliations{
\textsuperscript{\rm 1}Université Paris-Saclay, CentraleSupélec, CNRS, L2S, Gif-sur-Yvette, France\\
\textsuperscript{\rm 2}Institut Polytechnique de Paris, Télécom SudParis, Palaiseau, France\\
\textsuperscript{\rm 3}Delft University of Technology, The Netherlands\\
\{mohammed.abdullah, salaheddine.elayoubi\}@centralesupelec.fr,\\
 G.Iosifidis@tudelft.nl, tijani.chahed@telecom-sudparis.eu
}

\def\c{\mathcal}

\def\c{\mathcal}

\def\sumT{\sum_{t=1}^T}
\usepackage{bibentry}

\begin{document}

\maketitle

\begin{abstract}
We study Constrained Online Convex Optimization with Memory (COCO-M), where both the loss and the constraints depend on a finite window of past decisions made by the learner. This setting extends the previously studied unconstrained online optimization with memory framework and captures practical problems such as the control of constrained dynamical systems and scheduling with reconfiguration budgets. For this problem, we propose the first algorithms that achieve sublinear regret and sublinear cumulative constraint violation under time-varying constraints, both with and without predictions of future loss and constraint functions. Without predictions, we introduce an adaptive penalty approach that guarantees sublinear regret and constraint violation. When short-horizon and potentially unreliable predictions are available, we reinterpret the problem as online learning with delayed feedback and design an optimistic algorithm whose performance improves as prediction accuracy improves, while remaining robust when predictions are inaccurate. Our results bridge the gap between classical constrained online convex optimization and memory-dependent settings, and provide a versatile learning toolbox with diverse applications.
\end{abstract}

\section{Introduction}

Online Convex Optimization (OCO) is the workhorse model for sequential decisions under adversarial uncertainty. In its basic version, a learner picks an decision $x_t$ from a convex set $\mathcal X$ at the start of each round $t$; an adversary then reveals a convex loss function \(f_t\!:\!\mathcal X\!\mapsto\!\mathbb{R}\) and the learner suffers \(f_t(x_t)\). The learning is assessed by the metric of regret $\mathcal R_T$, i.e., the distance of the accumulated loss from that of the best-in-hindsight decision $x^\star\!=\!\arg\min_{x\in\mathcal X} \sum_{t=1}^{T} f_t(x)$, and the goal is to ensure sublinear regret. Since its conception \cite{zinkevich}, OCO has been extended to an impressive range of problems \cite{orabona2019modern}.

One of these extensions is OCO with memory (OCO-M), where the loss at each round $t$ depends on the previous $m$ decisions of the learner $x_{t-m}, \ldots, x_t$. Data caching with fetching costs, communication systems with reconfiguration delays, user-engagement in recommender systems, investment portfolio selection, and model training in continual learning, are only some of the problems that can be tackled with OCO-M \cite{anava2015online}. Further, the recent online non-stochastic control (NSC) framework owes its success to OCO-M, as such stateful systems can be controlled with memory-based functions \cite{hazan2025introductiononlinecontrol}.

Beyond minimizing losses, real-world systems must often satisfy average constraints of time-varying functions, $g_t(x_t)\!\leq\! 0$. This constrained OCO (COCO) extension is attracting growing interest and has several flavors. In some cases the goal is to ensure sublinear long-term violation (LTV), $\sum_t g_t(x_t)$, and in others to bound the cumulative constraint violation (CCV), $\sum_t \max\{g_t(x_t),0\}$; also, the functions may be known when $x_t$ is decided or not, and they might be static $g_t\!=\!g$, stochastically-perturbed or selected by an adversary. Importantly, similarly to losses, the constraints may exhibit memory and depend jointly on the $m$ recent decisions. Examples include energy budget constraints over $m$-slot windows in smart grid, thermal envelopes that integrate recent power inputs in processors, QoE user metrics capturing service volatility, and battery-health limits tied to cumulative depth-of-discharge, to mention only few. Apart from such operational or resource constraints, LTV and CCV are also relevant for multi-criteria optimization.

It is clear from the above that COCO-M is an important and practical extension of OCO which, nevertheless, remains largely unexplored. This work contributes in addressing this gap by studying several instances of this problem.

\textbf{Contributions.} We study the most compounded version of COCO-M where the constraints are unknown and adversarially varying, and we are interested in CCV, which we denote $\mathcal V_T$. We consider two problems, one with memory effects on losses and constraints (\texttt{COCO-M}$^2$), and another with memory-less constraints (\texttt{COCO-M}). Following a penalty-based relaxation analysis \cite{zangwill} we design an algorithm that achieves regret
$\mathcal R_T\!=\!\mathcal{O}\!\left(m^{3/2}\sqrt{T\log T}\right)$, and CCV $\mathcal V_T\!=\!\mathcal{O}\bigl(\max\{T^{3/4},\,m^{3/2}\sqrt{T\log T}\}\bigr)$ for \texttt{COCO-M}$^2$, improving upon the \(\mathcal R_T, \mathcal V_T\!=\!\mathcal{O}(T^{2/3}\log^2\!T)\) bounds of \cite{liu2023online} for $T \in [3,10^{49}]$, the only prior work related to \texttt{COCO-M}$^2$. For \texttt{COCO-M}, we achieve $\mathcal V_T\!=\! \mathcal O(T^{3/4})$ which improves to $\mathcal V_T\!=\!\mathcal O(m^{3/2}\sqrt{T}\log T)$ for short memory.

We take the next step and study, for the first time, the problem through the lens of \emph{optimistic learning} (OL), \cite{rakhlin2013optimization}. That is, we assume the availability of untrusted predictions about the gradients of forthcoming losses and constraints, and design an algorithm that, under a more restrictive benchmark than in the no-prediction setting, ensures
\begin{equation*}
\mathcal R_T\!=\! \mathcal{O}\left(\sqrt{\mathcal{E}_T(f)}\right),\ \mathcal V_T\!=\!\mathcal{O}\left(\left(\sqrt{\mathcal{E}_T(g^+)}\!+\!m\right) \log T\right), 
\end{equation*}
where \(\mathcal{E}_{T}(f)\) and \(\mathcal{E}_{T}(g^+)\) denote the total prediction errors for the loss and constraint functions over the $T$ rounds. These bounds diminish with the predictions' accuracy, becoming $\mathcal R_T\!=\!\mathcal O(\log T)$, $\mathcal V_T\!=\!\mathcal{O}(\!m \log{T})$ for perfect predictions, and $\mathcal{R}_T\!=\!\mathcal{O}(m^2\sqrt{T})$, $\mathcal{V}_T\!=\!\mathcal{O}(m^2\sqrt{T}\log(T))$ when the prediction fail maximally. These rates subsume the optimistic COCO bounds \emph{without memory} \cite{lekeufack2024optimistic}; and the optimistic \emph{unconstrained} OCO-M bounds \cite{mhaisen2024optimistic}.

To streamline the presentation of the material, we defer all proofs to the Appendix where, the interested reader, can also find extensive discussion of related work, analysis of special problem cases and numerical experiments.

\section{Related work}\label{sec:State-of-art}
\paragraph{COCO Bounds.}
The COCO literature falls in two strands.
First, works that assume constraints are \emph{static or known}. The earliest work here \cite{mahdavi2012trading}, considers fixed affine constraints; \cite{chaudhary2022safe} study fixed but unknown constraints observed via stochastic feedback; and \cite{qiu2023gradient}, \cite{yu2020low} address static unknown constraints to get LTV $\mathcal{O}(1)$. This regime is less relevant to our setting but provides useful insights. The second strand considers constraints that are both \emph{time-varying and unknown}. Here, \cite{guo2022online} match the \(\mathcal{O}\!\bigl(\sqrt{T}\bigr)\) regret and obtain \(\mathcal{O}(T^{\tiny{3/4}})\) CCV. Other papers, e.g., \cite{wang2025revisiting} focused on projection free algorithms for this problem.
\cite{sinha2024optimal} achieve \(\mathcal{O}\!\bigl(\sqrt{T}\bigr)\) regret with \(\mathcal{O}\!\bigl(\sqrt{T}\log T\bigr)\) CCV. For a dynamic benchmark, \cite{wang2025doubly} provide a bound of \(\mathcal{O}\!\bigl(T^{(1+V_x)/2}\bigr)\) and \(\mathcal{O}\!\bigl(T^{V_g}\bigr)\) violation, where \(V_x\) and \(V_g\in[0,1]\) quantify the functions variability. However, there are no OCO-M papers with either CCV or LTV.

\paragraph{COCO-M and Non-stochastic Control.} Since their introduction by \cite{agarwal2019logarithmic}, NSC methods have used disturbance–action (DAC) policies: the control at round~$t$ is a weighted sum of the last~$m$ disturbances, and these weights are learned through OCO‑M. In that sense NSC is relevant to our study. However, most NSC papers impose \emph{deterministic} constraints on the state and input~\cite{jiang2025online,li2021online,nonhoff2021online,yan2023online} or assume the constraint at ~$t{+}1$ is revealed one round before~\cite{zhou2023safe}; in both cases the goal is per\textendash round feasibility rather than an average\textendash sense guarantee.  
The sole exception is \cite{liu2023online}, who analyze a fully adversarial setting and obtain regret and CCV bounds of $\mathcal{O}\!\bigl(T^{2/3}\log^{2}T\bigr)$ when the memory length is fixed to $m=\log T$.  
Our work tightens these bounds to $\mathcal{O}\!\bigl(m^{3/2}\sqrt{T\log T}\bigr)$ regret and $\mathcal{O}\!\bigl(m^{3/2}\sqrt{T\log T}\,\vee\,T^{3/4}\bigr)$ CCV. Besides, our bounds hold for any memory length $m$.

\paragraph{Optimism.} Look-ahead gradients can compress regret in proportion to their prediction error. In OCO, this is well studied \cite{rakhlinNeurips13,mohri2016accelerating,JOULANI2020108,flaspohler2021online}. 
Extending optimism to OCO-M is harder because one must predict farther than the next round; \cite{mhaisen2024optimistic} is the only work that tackles this challenge for linear losses in NSC under \emph{imperfect} predictions. 
In online control, perfect look-ahead predictions can yield exponential improvements in dynamic regret. For example, Yu et al. \cite{yu2020power} analyze quadratic, time-invariant losses with adversarial disturbances under model-predictive control, and Li et al. \cite{li2019online} study time-varying convex losses without disturbances. In both cases, the dynamic-regret bound decreases exponentially with the length of the prediction window. Predictions may also be viewed through the lens of “context” \cite{li2022robustness} in stochastic MDPs with presume finite states and action sets. Lastly, prior works \cite{yu2022competitive,zhang2021regret} assume full‑horizon predictions, blocking the learner from using updates; we instead let it incorporate the latest forecasts each round.

Finally, OL in COCO remains surprisingly sparse. \cite{anderson2023lazy} proposed a primal-dual algorithm to achieve $\mathcal R_T\!=\!\mathcal{O}(1)$ and LTV $\mathcal{O}\!\bigl(\sqrt T\bigr)$ under perfect predictions;
\cite{ijcai2025p776} achieves $\mathcal{R}_T=\mathcal{O}(\sqrt{V_T})$, where $V_T$ captures the aggregate variation of successive gradients, and grantees \(\mathrm{LTV}=\mathcal{O}(1)\) under the Slater condition,
while \cite{lekeufack2024optimistic} used instead a penalty method, attaining $\mathcal R_T\!=\!\mathcal{O}\!\bigl(\sqrt{\bar{\mathcal{E}}_{T}(f)} \bigr)$ and CCV $\mathcal{O}\!\bigl(\log T(\sqrt{\bar{\mathcal{E}}_{T}(g^+)}  +1)\bigr)$, where $\bar{\mathcal{E}}_{T}(f)$ and $\bar{\mathcal{E}}_{T}(g^+)$ denote the prediction errors. None of these works consider memory in the objective or constraints.

\paragraph{Summary.} OCO-M with adversarial time-varying cumulative constraints has only been studied in NSC by \cite{liu2023online}, which we strictly outperform for $T \in [1,10^{49}]$. The optimistic version of the problem is introduced by this work, and we obtain bounds which for $m\!=\!0$ match the OL COCO results \cite{lekeufack2024optimistic}. A summary of the most relevant COCO results is provided in the appendix (Table \ref{tab:related-work}).

\section{Preliminaries}\label{sec:prelimi}

\textbf{Notation.} The diameter of a non-empty, closed and convex decision set \(\mathcal{X}\subset\mathbb{R}^{d}\),
is defined as $\|\mathcal{X}\|\!\doteq \sup_{x,y\in\mathcal{X}} \|x\!-\!y\|$, where \(\|\cdot\|\) is the $\ell_2$ norm.
We write $\c T\!:=\!\{m,\dots,T\}$ and at each round $t\in \c T$ the learner selects an decision $x_t\in \c X$. The \emph{memory length} is denoted with $m$, and we use:
\begin{equation}
x_{t-m}^{t}\;\doteq\;(x_{t-m},\dots,x_t)\in \c X^{m+1}, \quad x_{a:b}\doteq\sum_{i=a}^{b}x_i.\notag
\end{equation}
For a function $f_t(x_{t-m},\dots,x_t)$ we define its memory-less version $\hat{f}_t(x_t)\!\doteq\!f_t(x_{t},\dots,x_t)$ with $\hat{f}_t:\!\mathcal{X}\mapsto\! \mathbb{R}$, and its prediction is denoted by $\tilde{f_t}$. We use $f_t^+(x)\!\doteq\! \max\big\{f_t(x),0\big\}$, and abuse notation to denote $\nabla f_t(x_t)$ the gradient of $f_t$ at $x_t$ or its subgradient if it is non-differentiable.

\paragraph{Background.} In OCO the learner selects an decision \(x_t\in\mathcal{X}\) before the convex loss \(f_t:\mathcal{X}\!\mapsto\!\mathbb{R}\) is revealed. The performance of the learner is measured with the regret:
\begin{align}
\texttt{OCO}: \quad	\c R_T=\sumT f_t(x_t)-\min_{x\in \c X} \sumT f_t(x),
\end{align}
and the learner wishes to achieve $\lim_{T\rightarrow \infty} \c R_T/T=0$, for any possible sequence of loss functions $\{f_t\}_t$.

In a more recent extension of this framework, the learner's decisions need additionally to satisfy a time-average (budget) of constraints $g_t:\mathcal{X}\mapsto \mathbb{R}$ $\forall t$. In this Constrained OCO (COCO) formulation, the regret is defined as:
\begin{align}
\c R_T^c=\sumT f_t(x_t)-\min_{x\in \c X_T} \sumT f_t(x),
\end{align}
where the set of eligible decisions is modified to: 
\begin{align}
	\c X_T=\Big\{ x\in \c X \ \Big \vert \ g_t(x)\leq 0, \forall t\in \c T \Big\}.
\end{align}
Observe that we restrict the benchmark to satisfy the constraints at every round and not on average; a necessary compromise to avoid the impossibility result of COCO, cf. \cite{mannor2009online}. The learner here aims to achieve sublinear regret \emph{and} constraint violation:
\begin{align}
\texttt{COCO}: \ \		\c R_T^c=o(T), \   \ \ \c V_T^c\triangleq\sumT g_t^+(x_t)=o(T).
\end{align}

In this work we are interested in functions with $m$-length memory where the loss $f_t:\mathcal X^{m+1}\!\mapsto\! \mathbb R$ at each round $t$, depends on the previous $m\!>\!0$ decisions of the learner. Following \cite{anava2015online}, the regret for this OCO-M problem is defined as: 
\begin{align}
\!\!\texttt{OCO-M}:	\c R_T^m\!=\!\sum_{t=m}^T f_t(x_{t-m}^t) \!-\! \min_{x\in \c X}\sum_{t=m}^T\! f_t(x, \ldots,x),
\end{align}
where the benchmark is defined using the respective \emph{memory-less} functions $\hat f_t(x)\doteq f_t(x,\ldots, x)$ that are assumed convex. In this work, we make a further step and introduce the $\texttt{COCO-M}^2$ framework, where the constraints also exhibit memory effects, $g_t:\mathcal{X}^{m+1}\mapsto \mathbb{R}$. We thus define: 
\begin{align}
\quad \c R_T^{mc}=\sum_{t=m}^T f_t(x_{t-m}^t) - \min_{x\in \c X_T^m}\sum_{t=m}^T f_t(x, \ldots,x),
\end{align}
where the set of eligible decisions is:
\begin{align}
	\c X_T^m=\Big\{ x\in \c X \ \Big \vert \ g_t(x,\ldots,x)\leq 0, \forall t\in \c T \Big\}
\end{align}
and, as in the typical COCO, the learner aims to achieve:
\begin{align}\label{coco-m2}
\texttt{COCO-M}^2: \ \c R_T^{mc}, \ \c V_T^{mc}\doteq\sumT g_t^+(x_{t-m}^t)=o(T).
\end{align}
Finally, the COCO setting where only the loss functions have memory while the constraints, $g_t:\mathcal X\mapsto \mathbb R$, are memoryless, is of independent interest. In this case, the definition of regret  $\c R_T^{mc}$ remains the same, but the benchmark is $x^\star \in \c X_T$, and the constraint violation changes to:
\begin{align}\label{coco-m}
\texttt{COCO-M}:& \ \c R_T^{mc}, \ \c V_T^{c}\doteq\sumT g_t^+(x_t)=o(T).
\end{align}
We denote this problem as $\texttt{COCO-M}$ to distinguish it from the above problem with \emph{double} memory ($\texttt{M}^2$).

Regarding the solution of COCO problems, the main techniques use a simple idea: apply an OCO algorithm on some type of (time-varying) Lagrange function, $\c L_t(\cdot)$, that scalarizes the objective and constraints. These techniques can be classified in two broad categories. Those that introduce explicit dual variables $\mu$ and perform primal-dual iterations on $\mathcal{L}_t(x, \mu)$, i.e., employ coordinated learning in the primal and dual space, e.g., see \cite{lamperski, valls}. And the second category that draws from penalty methods \cite{zangwill} and creates again a Lagrange-type function, $\c L_t(x)$, where the constraint violation is penalized with some parameter \cite{pmlr-v97-liakopoulos19a,doug-ecml22, lekeufack2024optimistic, sinha2024optimal}. We follow this latter approach.

\paragraph{Learning Model \& Assumptions.} We consider the most general COCO model where both the loss and the constraint functions may change over time, and in each round they are revealed \emph{after} the learner commits its decision. Regarding the adversary model, we follow \cite{anava2015online,merhav2002sequential,gyorgy2014near} and consider an oblivious adversary which implies that these functions are determined in advance (i.e., at $t=0$) but, of course, are not revealed. Finally, we consider full-information feedback for all the arguments of the memory functions; see Fig. \ref{fig:model}. We also use the following standard OCO assumptions.
\begin{assumption}\label{ass:domain}
	$\c X$ is closed and convex, with $\|\c X\|<\infty$.
\end{assumption}

\begin{assumption}\label{ass:convex}
	For every $t$, functions ${f}_t(\cdot)$ and ${g}_t(\cdot)$ are convex and $F$-, $G$-bounded, respectively. 
\end{assumption}

\begin{assumption}[Lipschitz continuous ]\label{ass:lipschitz}
For every time $t\in \c T$, let 
$f_t, g_t:\mathcal{X}^{m+1}\!\mapsto\mathbb{R}$ 
and  
$\tilde f_t,\tilde g_t:\mathcal{X}^{m+1}\!\mapsto\mathbb{R}$.

\begin{figure}[t]
\begin{center}
\includegraphics[width=0.8\columnwidth]{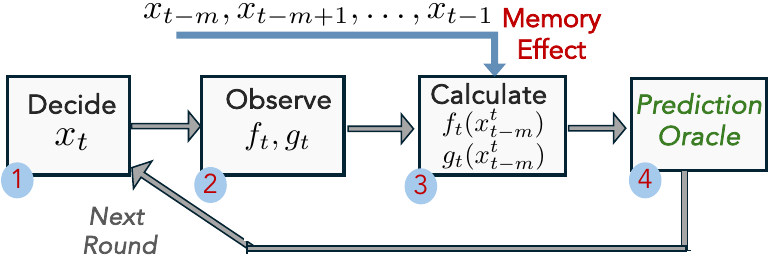}
\end{center}
\vspace{-2mm}
\caption{{decision stages of \texttt{COCO-M} (with predictions).}} 
\label{fig:model}
\end{figure}

All functions are Lipschitz continuous, i.e., there exist finite constants $L_{t,f},L_{t,g}\ge0$ such that,
$\forall \mathbf{x,y} \in\mathcal X^{m+1}$,
\begin{align*}
\Bigl| f_t(\mathbf x)- f_t(\mathbf y)\bigr|,\;
\bigl|\tilde f_t(\mathbf x)-\tilde f_t(\mathbf y)\bigr|
&\le L_{f}\,\|\mathbf x- \mathbf y\|,\\[4pt]
\bigl| g_t(\mathbf x)- g_t(\mathbf y)\bigr|,\;
\bigl|\tilde g_t(\mathbf x)-\tilde g_t(\mathbf y)\bigr|
&\le L_{g}\,\|\mathbf x-\mathbf y\|.
\end{align*}
\end{assumption}

Finally, the following remark is without loss of generality.
\begin{remark}
    If we have $n$ constraints $g_{t,k}$ with $k=1,\ldots,n$, we can follow the standard approach and define $g_t:=\max_{k}g_{t,k}$. Hence, to streamline the presentation, we consider constraints that map to $\mathbb R$, not to $\mathbb R^n$. 
\end{remark}

\section{Algorithm Design \& Analysis}\label{sec:COCO with memory without predic}

Our method is closer to penalty COCO techniques, while the memory effect is handled as in \cite{anava2015online}, i.e., we perform the analysis using the memory-less functions and lift the results to the original problem. Namely, we use the memory-less versions of \(f_t,g_t:\c X^{m+1}\!\mapsto \mathbb{R}\), which are defined on $\mathcal{X}$ as:
\[
\hat f_t(x):=f_t(x,\dots,x),
\qquad
\hat g_t(x):=\;g_t(x,\dots,x),
\]
and following \cite{lekeufack2024optimistic,sinha2024optimal}, we define the (memory-less) surrogate function:
\begin{equation}\label{eq:surrogate}
	\hat{\mathcal{L}}_t(x)
	\;=\;
	\hat f_t(x)
	\;+\;
	\Phi'\!\bigl(\hat V_{t}\bigr)\,
	\hat g_t^{+}(x),
\end{equation}
where $\Phi:\!\mathbb R_+\!\mapsto\! \mathbb R$ is a non-negative, convex monotonically increasing \emph{penalty} function with $\Phi(0)\!=\!0$, and $\Phi^\prime$ its derivative. Observe that \(\Phi\) is applied to the cumulative memory-less constraint violation which is denoted \(\hat{V}_t\) and defined: 
\[
  \hat{V}_t=
  \begin{cases}
    \hat{\mathcal{V}}^{\,\mathrm{mc}}_t, & \text{in }\texttt{COCO-M}^2,\\
 \mathcal{V}^{\,\mathrm{c}}_t,      & \text{in }\texttt{COCO-M}.
  \end{cases}
\]
This quantity evolves with time as:
\[
  \hat{V}_t=
  \begin{cases}
    \hat V_{t-1}+\hat g_{t}^{+}(x_{t}), & \text{in }\texttt{COCO-M}^2,\\
\hat V_{t-1}+ g_{t}^{+}(x_{t}),      & \text{in }\texttt{COCO-M}.
  \end{cases}
\]
The intuition of  $\hat{\mathcal{L}}_t$ is that we penalize the violation of constraints at round $t$, with a penalty commensurate to the accumulated constraint violation. Different from \cite{lekeufack2024optimistic}, \cite{sinha2024optimal}, we use: 
$$\Phi(\mathrm{V})=\lambda V^2,$$ 
with parameter $\lambda\!>\!0$ determined below. This penalty leads to tighter bounds compared to the exponential function in those prior works; we revisit this discussion later.  

Now, focusing on the surrogate function, we observe that it operates on decisions drawn from the convex set $\c X$, and under Assumptions 1-3 it is convex with bounded gradients for fixed $T$. Namely, by the triangle inequality, we get:
\begin{equation}\label{eq:grad_bound}
\sup_{t,x}\bigl\|\nabla\hat{\mathcal{L}}_t(x)\bigr\|\;\le\; L_f + \Phi'(V_T) L_g=L_f+ 2\lambda V_T L_g.
\end{equation}
Therefore, performing Online Gradient Descent (OGD), i.e.,
\begin{equation}\label{eq:COCO-update}
	x_{t+1}	= \mathcal{P}_{\mathcal{X}} \Big(x_t - \eta_t \nabla \hat{\mathcal{L}}_t(x_t) \Big),
\end{equation}
where $\mathcal P_{\mathcal X}$ is the $\ell_2$ projection on $\mathcal X$, ensures sublinear regret for these surrogate functions, which we denote $\mathcal{R}_T(\hat{\mathcal{L}})$.

\begin{lemma}\label{lem:OGD_adaptive_bound}
Under Assumptions 1-3, performing OGD on \eqref{eq:surrogate} with step $\eta_t\!=\!\frac{\sqrt{2}\|\mathcal{X}\|}{2\sqrt{\sum_{\tau=1}^{t} \|\nabla \hat{\mathcal{L}}_t(x_t)\|^2}}$, we obtain:
\begin{align*}
\c R_T(\hat{\c L}) &\triangleq \sum_{t=1}^{T}\hat{\mathcal{L}}_t(x_t) - \min_{x\in\mathcal{X}_T^{\mathrm{b}}} \sum_{t=1}^{T}\hat{\mathcal{L}}_t(x)\\& \leq \sqrt{2}\| \mathcal{X}\| \sqrt{\sum_{t=1}^{T} \| \nabla \hat{\c L}_t(x_t)\|^2}
\end{align*}
where the benchmark set depends on the problem:
\(\mathcal{X}_T^{\mathrm{b}}=\mathcal X_T^{m}\) for \texttt{COCO-M}$^2$
and \(\mathcal{X}_T^{\mathrm{b}}=\mathcal X_T\) for \texttt{COCO-M}.
\end{lemma}
In the sequel, we explain how this result sets the basis for tackling the $\texttt{COCO-M}$ and $\texttt{COCO-M}^2$ problems.

\subsection{Problem  $\texttt{COCO-M}^2$: Double Memory Effect} 

We start with the more compounded problem $\texttt{COCO-M}^2$ in \eqref{coco-m2}, having memory in losses and constraints. Without loss of generality, we assume these effects are of the same length $m$. The solution is summarized in Algorithm \ref{algo:memory-coco}. In brief, in each round $t$, the learner commits its decision $x_t\in \c X$, observes the current loss and constraint functions and calculates the loss and constraint violation. Then, it updates the constraint violation, calculates the $t$-round gradient $\nabla \hat{\c L}_t(x_t)$ and uses OGD to devise the next decision.

To characterize the performance of this algorithm, we take two steps. First, we utilize the following \emph{regret decomposition} result, which links the regret of the memory-less Lagrangian $\c R_T(\hat{\c L})$ to the regret of the memory-less loss (denoted $\hat{\c R}_T^c$) and the constraint violation $\hat{{\c V}}_T^{mc}$ over the memory-less functions. 

\begin{lemma}[Regret decomposition \cite{sinha2024optimal}]\label{lem:Regret_decom_no_predic}
For any OCO algorithm, if $\Phi$ is a convex increasing function, we have for any $t\geq m$ and $x^\star \in \mathcal{X}_{T}^{m}$
    \begin{equation}
        \hat{\c R}_T^c + \Phi(\hat{V}_T) -\Phi(\hat{V}_m) \leq \mathcal{R}_T(\hat{\c L}).
    \end{equation}
\end{lemma}

\begin{algorithm}[t]
	\caption{\small{Learning for \textsc{$\texttt{COCO-M}$} and \textsc{$\texttt{COCO-M}^2$}}}
	\label{algo:memory-coco}
\begin{small}    
	\begin{algorithmic}[1]
		\REQUIRE initial history $x_{0}^{m-1}\in X^{m}$; dual seed $\hat{\mathrm{V}}_{0}^{m-1}\gets\mathbf 0$
		\FOR{$t=m$ \TO $T-1$}
		\STATE Play $x_t$ and observe $f_t(\cdot),\,g_t(\cdot)$
		\STATE Calculate $f_t(x_{t-m}^t)$ and $g_t(x_t)$. \qquad {\color{blue}{\textit{\small{// for \ $\texttt{COCO-M}$}}}}
		\STATE Calculate $f_t(x_{t-m}^t)$, $g_t(x_{t-m}^t)$. \quad {\color{blue}{\textit{\small{// for \ $\texttt{COCO-M}^2$}}}}
		 \STATE ${\mathrm{V}}_{t}\gets{\mathrm{V}}_{t-1}+g_t^{+}(x_{t})$\quad {\color{blue}{\textit{\small{// dual update for $\texttt{COCO-M}$}}}}
		\STATE $\hat{\mathrm{V}}_{t}\gets\hat{\mathrm{V}}_{t-1}+g_t^{+}(x_{t},\dots,x_t)$\quad {\color{blue}{\textit{\small{// for $\texttt{COCO-M}^2$}}}}
        \STATE Compute surrogate gradient $\nabla \hat{\c L}_t(x_t)$ via~\eqref{eq:surrogate}
        \STATE $x_{t+1}\gets$ solution of~\eqref{eq:COCO-update} \qquad {\color{blue}{\textit{\small{// devise next decision}}}}
		\ENDFOR
	\end{algorithmic}
\end{small}    
\end{algorithm}

And, secondly, we transfer these results to the original problem with memory, by observing that we can write:
\begin{align}
	\c{R}_{T}^{c}
	&= \underbrace{\sum_{t=m}^{T}
		f_t(x_{t-m}^t)-f_t(x_t,\dots,x_t)}
	_{\text{memory deviation}}+
	\hat{\c{R}}_{T}^c\label{eq:objective_memory_decom}\\
   \mathcal{V}_T^{mc} &= \sum_{t=m}^{T} g_t^+(x_{t-m}^t) - g_t^+(x_t,\dots,x_t) + \hat{\mathcal{V}}_T^{mc}. \label{eq:constraint_memory_decom}
\end{align}
We bound this memory deviation by exploiting the functions Lipschitz continuity together with the one-step OGD bound. The following theorem summarizes the achieved bounds.

\begin{theorem}[Regret and CCV for $\texttt{COCO-M}^2$]
	\label{thm:full-memory}
	Assume \textnormal{(i)}~Assumptions
	\ref{ass:domain}, \ref{ass:convex}, \ref{ass:lipschitz} hold;
	\textnormal{(ii)} the constraint is
	$g_t(x_{t-m}^{t})$,  
	and \textnormal{(iii)} we use the step
	\(
	\eta_t\!=\! \frac{\sqrt{2}\|\mathcal{X}\|}{2\sqrt{\sum_{\tau=1}^{t} \|\nabla \hat{\mathcal{L}}_t(x_t)\|^2}}
	\)
	and the penalty 
	\(
	\Phi(V)\!=\!\lambda V^2
	\)
	with
	\(
	\lambda\!=\!\tfrac{1}{\sqrt{T}}.
	\)
Then, $\forall \ T\!\ge\! m$ it is:
	\begin{align}
		\c{R}_{T}^{mc}
		&=\mathcal{O}\Big(m^{\frac{3}{2}} \sqrt{T\log(T)}\Big)
		\\
		\c{V}_T^{mc}
		&=\mathcal{O}\left(\max\{T^{3/4},m^{\frac{3}{2}}\sqrt{T\log (T)}\right).
	\end{align}
\end{theorem}

\paragraph{Discussion.}
In Theorem~\ref{thm:full-memory} we see the effect of $m$: longer memory amplifies both the regret and CCV, thus the
performance degrades as the window grows. The only directly comparable study is \cite{liu2023online}, which assumes $m\!=\!\log T$ and achieves strictly inferior bounds (Table \ref{tab:related-work}) for any $T\!\in\![3,10^{49}]$.
When memory vanishes (\(m\!=\!0\)), our bounds collapse to
\(\mathcal{O}(\sqrt{T})\) for the regret and
 \(\mathcal{O}(T^{3/4})\) for the CCV; which are looser than the bounds of \cite{sinha2024optimal}. This gap arises because we employ a \emph{quadratic} penalty that yields sharper results for the memory problem, whereas ~\cite{sinha2024optimal} use an exponential penalty in the memory–free setting. We will show later that, for $\texttt{COCO-M}$ and sufficiently short memory, this gap vanishes as we also use there a different penalty function.

\subsection{Problem  $\texttt{COCO-M}$} 
We continue with the case where memory affects only the losses, while the constraints are time-varying memory-less functions. The goal is to ensure sublinear regret $\c R_T^{mc}$ and constraint violation $\c V_T^c$, see \eqref{coco-m}. The algorithm is identical to the one above, with only changes in the calculation of the constraint violation and the dual update. In particular, following the same steps, we first invoke the regret–decomposition lemma and translate the bound back to the original problem using solely \eqref{eq:objective_memory_decom}, because, here, the term \(\hat{V}_T\) appearing in the lemma coincides with the cumulative violation \(\mathcal{V}_T^{c}\). We thus get the next result.

\begin{theorem}[Regret and CCV with memory-free constraint]
	\label{thm:obj-memory}
	Given that \textnormal{(i)}~Assumptions
	\ref{ass:domain}, \ref{ass:convex}, and \ref{ass:lipschitz} hold,
	\textnormal{(ii)} the constraint is memory-less
	$(g_t=(x_t))$,  
	and \textnormal{(iii)} we use the adaptive step
	\(
	\eta_t= \frac{\sqrt{2}\|\mathcal{X}\|}{2\sqrt{\sum_{\tau=1}^{t} \|\nabla \hat{\mathcal{L}}_t(x_t)\|^2}}
	\)
	and the penalty
	\(
	\Phi(V)=\lambda V^2
	\)
	with
	\(
	\lambda=\tfrac{1}{\sqrt{T}}.
	\)
Then, for any $T\!\ge\! m$ we have:
	\begin{align}
		\c{R}_{T}^{mc}=\mathcal{O}\Big(m^{\frac{3}{2}} \sqrt{T\log(T)}\Big),
		\ \ 
		\c{V}_T^{c}=\mathcal{O}\left(T^{3/4}\right).
	\end{align}
\end{theorem}

\paragraph{Discussion} 
Relatively to $\texttt{COCO-M}^2$ (Theorem~\ref{thm:full-memory}), two points stand out:
\textnormal{(i)} The CCV no longer depends on the window \(m\); it remains
      \(\mathcal{O}(T^{3/4})\) regardless of the loss functions' memory. \textnormal{(ii)} The regret retains the \(m^{3/2}\sqrt{T\log T}\) factor. Compared to the regret bound \(\mathcal{O}\!\bigl(m^{3/2}\sqrt{T}\bigr)\) of \cite{anava2015online},
this carries an extra \(\sqrt{\log T}\) factor; this is the price of enforcing time–varying constraints. 

Furthermore, it is important to stress that when the memory length satisfies
\(m\!\le\!\bigl(T^{1/6}/\log T\bigr)^{1/3}\), thus
\(m^{3/2}\sqrt{T\log T}\!\le\! T^{3/4}\), we can replace the quadratic penalty with an exponential one with a tuned~\(\lambda\), to achieve tighter guarantees for CCV: $\mathcal{R}_T^{mc}
  =\mathcal{O}\!\bigl(\sqrt{T}+m^{3/2}\sqrt{T\log T}\bigr)$ and $\mathcal{V}_T^c
  =\mathcal{O}\!\bigl(\sqrt{T\log T}+ m^{3/2}\sqrt{T}\log T\bigr).$
While $\mathcal{V}_T^c$ depends now on $m$, this bound is smaller than $\mathcal{O}(T^{3/4})$. In several practical applications indeed $m$ is a constant and much smaller than any expression of the growing $T$. Hence, for all these problems we can enable these improved rates. Finally, we note that for \(m\!=\!0\) the bounds reduce to those in \cite{sinha2024optimal}. We provide the details for these cases in the Appendix.

\section{Benefiting from Predictions}\label{sec:COCO with memory and prediction}
 
We next study \texttt{COCO-M}$^2$ when predictions about forthcoming loss and constraint functions are available see Figure \ref{fig:model} for the decision stages. We study problem \texttt{COCO-M} in the Appendix. This form of learning, known as \emph{Optimistic Learning} (OL), achieves bounds that shrink with the initially-unknown predictions' accuracy. Predictions are widely studied in OCO \cite{rakhlinNeurips13}, less so in COCO \cite{lekeufack2024optimistic}, and only recently in OCO-M \cite{mhaisen2024optimistic} -- still, without constraints. In classical OCO, it suffices to predict the next gradient; alas in the presence of memory the forecasting should involve the gradients of the next $m$ slots.  

This compounded problem requires modifying the approach in Sec. \ref{sec:COCO with memory without predic}. First, we use the \emph{memory-based} surrogate:
\begin{equation}
\!\mathcal L_t\bigl(x_{t-m}^{\,t}\bigr)
\;=\;
f_t\bigl(x_{t-m}^{\,t}\bigr)
\;+\;
\Phi'\!\bigl(V_{t-m-1}\bigr)\,g_t^{+}\bigl(x_{t-m}^{\,t}\bigr), \label{eq:memory-surrogate}
\end{equation}
with \(\Phi(V)\!=\!\exp(\lambda V)\!-\!1\) that has delayed argument $V_{t-m-1}$. We study the \texttt{COCO-M}$^2$ problem, where:
\[
V_t = V_{t-1}+g_t^{+}\!\bigl(x_{t-m}^{\,t}\bigr),
\quad \text{i.e., }
V_T =\mathcal V_T^{mc}.
\]
Secondly, instead of performing OL on \eqref{eq:memory-surrogate}, we turn the problem on its head and interpret the losses and constraints as having {delayed gradients} instead of depending on past decisions: at round $t$ the learner selects $x_t$ but is only able to observe the gradient of this decision at the end of $t+m$, i.e., after all functions influenced by $x_t$ are revealed. This change of vantage point allow us to replace the memory effect with a delay effect, which then is handled through a particular version of OL. Before we proceed, we need the following.
\begin{assumption}[Separability]\label{ass:separable}
Every memory-based function can be decomposed into components, each
depending on a decision from a different round:
\[
f_t\bigl(x_{t-m}^{t}\bigr)=\sum_{i=0}^{m} f_t^{i}(x_{t-i}),
\qquad
g_t\bigl(x_{t-m}^{t}\bigr)=\sum_{i=0}^{m} g_t^{i}(x_{t-i}).
\]
\end{assumption}
\noindent where $f^i_t$, $g_t^i$ are defined on $\mathcal{X}$, and $i$ marks that their argument was decided in round $t-i$.
\begin{assumption}[Linearity]\label{ass:linearity}
Every function \(f_t^{i}(\cdot)\) and \(g_t^{i}(\cdot)\) is linear, for
all \(t\in \mathcal{T}\), \(i\in [0,m]\).
\end{assumption}


\begin{figure}[t]
\begin{center}
\includegraphics[width=0.65\columnwidth]{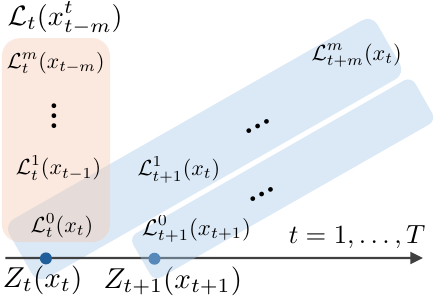}
\end{center}
\vspace{-1mm}
\caption{\textbf{Diagonal}: $Z_t(x_t)$ depends only on $x_t$, but includes loss/constraint components from next $m$ rounds. \textbf{Vertical}: $\mathcal L_t(x_{t-m}^t)$ includes function components only from round $t$, but depends on all past $m$ decisions.} 
\label{fig:forward-function}
\end{figure}

Now, the key idea for pivoting to delayed learning is introducing a \emph{forward} loss function to capture the influence of each $x_t$ on the system operation, i.e., $\ Z_t(x_t)\!\doteq \!$
\[
  \!\sum_{i=0}^{m}\mathcal L_{t+i}^i(x_{t})
  \!\triangleq\!
  \sum_{i=0}^{m} f_{t+i}^{\,i}(x_t)
  \!+\!
  \Phi'(V_{t-m-1+i})\,g_{t+i}^{i,+}(x_t).
\]
Despite the similarities with $\mathcal L_t\bigl(x_{t-m}^{\,t}\bigr)$, $Z_t$ depends only on round $t$ decision (is memory-less) and includes function components from the next $m$ rounds (has delayed gradient), see Figure \ref{fig:forward-function}). Our strategy is to bound the regret of $Z_t$ and translate that bound to the initial problem. A similar construction was used for OCO in \cite{mhaisen2024optimistic}.

The next new decomposition lemma ties the regret and CCV to memory-based surrogate loss and forward function.
\begin{lemma}\label{lem:memory-prediction-forward-decomposition}
For any valid penalty function $\Phi$ and under Assumption\footnote{We trivially set $f_t(\cdot)=g_t(\cdot)=0$ for $t\le m$ or $t>T$.} \ref{ass:separable} it holds:
\begin{align}
&\Phi\!\bigl(V_{T}\bigr)-\Phi\!\bigl(V_{m-1}\bigr)
      +\mathrm{R}_T^{mc}\notag \\
  &\;\le\;
    \mathcal{R}_{T}(\mathcal{L})
    +G(m+1)\,\Phi'\!\bigl(V_{T}\bigr)
    \label{eq:lemma-main-1}\\
  &\;\le\;
    \mathcal{R}_{T}(Z)
    +G(m+1)\,\Phi'\!\bigl(V_{T}\bigr).
    \label{eq:lemma-main-2}
\end{align}
where the regret is defined for benchmarks in the set:
\[
  \mathcal X_T^{\mathrm{mp}}
  \;=\;
  \bigl\{\,x\in \mathcal X : g_t^{\,i}(x)\le0,\
          \forall t\in \c T,\ i\leq m\bigr\}.
\]
\end{lemma}

\noindent Equipped with this result, it suffices to bound $\mathcal{R}_T(Z)$, which we achieve by utilizing predictions to cope with its delayed gradients. In particular, we leverage \cite{flaspohler2021online} that proposed a suite of delayed–optimistic algorithms (without memory or constraints), and adapt their \emph{Optimistic Delayed AdaF} (\texttt{ODAF}) algorithm, which has the tightest guarantees, for our problem. \texttt{ODAF} relies on FTRL \cite{mcmahan2011follow} which decides the next decision using all the previous gradients. In the standard delayed-feedback OCO, the entire gradient $\nabla f_{t}(x_t)$ is revealed in one shot, exactly $m$ rounds after $x_{t}$ is decided. With a memory window of length~$m$ the situation is subtler: each decision~$x_t$ affects \emph{some of the components} of the loss functions in rounds $t,\dots,t+m$; hence, the gradient information arrives piecemeal and may be delayed up to~$m$ rounds before fully revealed. Specifically, at round~\(t\) the learner possesses:
\begin{itemize}
    \item[\(\checkmark\)] \emph{revealed} gradients
          \(\nabla Z_{\tau}(x_\tau)\), for $\tau=1,\dots,t-m-1$ (coming from \(x_1,\dots,x_{t-m-1}\));
    \item[\(\circ\)] \emph{delayed} gradients that still depend on
          \(x_{t-m},\dots,x_{t-1}\);
    \item[\(\times\)] \emph{unseen} gradients that will depend on decision \(x_t\).
\end{itemize}

With this in mind,  we perform OL using an oracle that provides the next forward function gradient:
\[
\nabla\tilde Z_t(\tilde{x}_t)
=\sum_{i=0}^{m}
      \Bigl[
          \nabla\tilde f_{t+i}^{\,i}(\tilde{x}_t)
        + \Phi'\!\bigl(V_{t-m-1+i}\bigr)\,
          \nabla\tilde g_{t+i}^{i,+}(\tilde{x}_t)
      \Bigr],
\]
as well as the missing past gradients, which we combine into a single \emph{hint} vector: \(h_t\doteq\)
\begin{align}
&
\underbrace{%
 \! \sum_{i=0}^{m-1}\!\Bigg(\!\!\sum_{j=0}^{m-i-1}
      \!\Bigl[
          \nabla f_{t-m+i+j}^{\,j}
        \!+\! \Phi'\!\bigl(V_{t+i+j-2m-1}\bigr)\!
          \nabla g_{t-m+i+j}^{j,+}
      \!\Bigr]}_{\textcolor{blue}{\text{available at $t$}}}\!+\!
\nonumber\\
&
\underbrace{%
 \!\! \sum_{j=m-i}^{m}\!\!
      \Bigl[\!
          \nabla\tilde f_{t-m+i+j}^{\,j}
        \!+\! \Phi'\bigl(V_{t+i+j-2m-1}\!\bigr)
          \nabla\tilde g_{t-m+i+j}^{j,+}
     \! \Bigr]\!\!\Bigg)
  \!+\! \nabla\tilde Z_t}_{\textcolor{red}{\text{future predictions}}} \notag
\end{align}
where we denote $\nabla\! Z_t(x_t)$ with $\nabla\!Z_t $, and $\nabla\!\tilde Z_t(\tilde{x}_t)$ with $\nabla\!\tilde Z_t$.

Similarly to other OL algorithms \cite{rakhlin2013optimization}, \texttt{ODAF} performs an update (here, using FTRL) whose regularization is scaled by the prediction error. After the decision is committed, the losses $f_t$ and $g_t$ are revealed, rendering available the forward function $Z_{t-m}(x_{t-m})$. At the end of $t$, the learner therefore knows the gradients $\nabla Z_\tau$ for all $\tau=1,\dots,t-m$ and can evaluate the error of the hint $h_{t-m}$, which covered the window $\tau=(t-2m)({:}t-m)$,
\begin{equation}\label{eq:prediction-error}
\epsilon_{t-m}(Z)=\left\|\sum_{\tau=t-2m}^{t-m}\nabla Z_{\tau}-h_{t-m}\right\|^{2},
\end{equation}
with $\mathcal E_T(Z)=\sum_{t=m}^{T}\epsilon_t(Z)$ denoting the cumulative prediction error. The loss-function prediction error is then:
\begin{align*}
    \epsilon_{t-m}(f)&=
\left\| \sum_{s=t-2m}^{t-m} \sum_{i=t-s}^{m} \left( \nabla \tilde{f}_{s+i}^{i}(\tilde{x}_s) - \nabla f_{s+i}^{i}(x_s) \right) \right\|^2,
\end{align*}
where the outer sum runs over the last \(m\) decisions \(x_s, s=t-m,\ldots,m\), whose delayed contributions have not been fully revealed at $t-m$, and the inner sum selects only those slices that arrive \emph{after} \(t-m\) (i.e., delays \(i\ge t-s\)), and measures the difference between their predicted and true gradients. Similarly, we can define the constraint-function error as $\epsilon_{t-m}(g^+)$, and we denote $\mathcal{E}_T(f)=\sum_{t=m}^{T}\epsilon_{t}(f)$ and $\mathcal{E}_T(g^+)=\sum_{t=m}^{T}\epsilon_{t}(g^+)$ the cumulative errors.


\begin{algorithm}[t]	\caption{{\small{Optimistic learning for \textsc{$\texttt{COCO-M}^2$}}}}
	\label{algo:memory-coco_predic}
\begin{small}    
	\begin{algorithmic}[1]
		\REQUIRE initial history $x_{0}^{m}\in \mathcal{X}^{m+1}$; dual seed $\hat{V}_{0}^{m-1}\gets\mathbf 0$
		\FOR{$t=m+1$ \TO $T-1$}
		\STATE Play $x_t$ and observe $f_t(\cdot),\,g_t(\cdot)$
		\STATE Calculate $f_t(x_{t-m}^t)$, $g_t(x_{t-m}^t)$. 
		\STATE $V_{t}\gets V_{t-1}+g_t^{+}(x_{t-m}^t)$ 
        \STATE Compute the prediction error $\epsilon_{t-m}$ as in \eqref{eq:prediction-error}.
        \STATE Compute the predictions $h_{t+1}$
        \STATE Decide decision:
        $x_{t+1}\leftarrow \texttt{ODAF}\left(\nabla Z_{1:t-m}, h_{t+1},\epsilon_{t-m}\right)$.
		\ENDFOR
	\end{algorithmic}
\end{small}    
\end{algorithm}

Having clarified the prediction and error calculations, we proceed to present the learning mechanism, which is summarized in Algorithm \ref{algo:memory-coco_predic}. At each round \(t\) the learner chooses an decision \(x_t\), observes the realized loss and constraint functions (line 3 ), and updates the multiplier \(\Phi'\!\bigl(V_{t-m-1}\bigr)\) for use at the next step. It then evaluates the prediction–error (line 5) and \emph{predicts} the forward loss \(h_{t+1}\) for the next round; finally, it feeds the \texttt{ODAF} routine (line 7) with the cumulative \emph{revealed} gradients \(\nabla Z_{1:t-m}\), the {hint}\(h_{t+1}\), and the prediction error \(\epsilon_{t-m}\), to find the next decision. Due to lack of space, the details for \texttt{ODAF} are deferred to Appendix. Essentially, using the hints and prediction errors designed specifically for our problem, one can readily call the algorithm from \cite{flaspohler2021online}. The next theorem establishes regret and CCV guarantees for this optimistic setting. 
\begin{theorem}\label{Theom:memory-prediction}
Under the following conditions:
\begin{itemize}
    \item Assumptions \ref{ass:domain}, \ref{ass:convex}, \ref{ass:lipschitz}, \ref{ass:separable} and \ref{ass:linearity} hold;
    \item The update rule is \texttt{ODAF};
    \item \( \Phi(V)\!=\!\exp(\lambda V)\!-\! 1 \), with
    \(
    \lambda = \frac{1}{2 \left( C \sqrt{\mathcal{E}_T(g^+)} + G(m+1) \right)},
    \)
\end{itemize}
the following bounds hold:
\begin{align}
    \mathcal{R}_T^{mc} &= \mathcal{O}\left( \sqrt{\mathcal{E}_T(f)} \right), \\
    \mathcal{V}_T^{mc} &= \mathcal{O}\left( \left( \sqrt{\mathcal{E}_T(g^+)}+ m \right) \log T \right).
\end{align}
\end{theorem}

\textbf{Discussion.} Let us start by noting that the value of~\(\lambda\) depends on the (unknown) prediction error \(\mathcal{E}_{T}(g)\), but we can adjust it online via a \emph{doubling trick} that adds only an extra \(\log T\) to the bounds, similar in spirit to \cite{lekeufack2024optimistic}; the full analysis is in Appendix. As prediction accuracy improves, the bounds tighten, where under perfect prediction, they reach $\mathcal{O}(1)$ regret and $\mathcal{O}(m\log T)$~CCV. On the other hand, even if the predictors fail completely, the algorithm guarantees regret \(\mathcal{O}(m^2\sqrt{T})\) and CCV \(\mathcal{O}(m^2\sqrt{T}\log T)\). These bounds are tighter than those of Section~\ref{sec:COCO with memory without predic}, but they refer to a more restrictive benchmark $\mathcal{X}_T^{mp}\!\subset\!\mathcal{X}_T^{m}$. What is more, the predictions accuracy does not need to be known in advance, and, further, our solution allows the oracle to update its forecast at every round and benefit from more accurate information whenever available. This flexibility is crucial as predictions can indeed improve with time. Finally, we note that these bounds include as special cases important prior works. When $m\!=\!0$, the rates coincide with those of \cite{lekeufack2024optimistic},
who study COCO with predictions but \emph{no} memory.  
For OCO-M, \cite{mhaisen2024optimistic}
obtain \(\mathcal{O}(1)\) regret under perfect predictions, relying on the delayed-feedback framework of \cite{flaspohler2021online}; our analysis yields the same bound when constraints are omitted. Our work extends these ideas to time-varying \emph{and} memory-dependent constraints. 

There are also some important notes in place regarding the Assumptions. First, observe that the surrogate function uses the delayed penalty $V_{t-m-1}$ because at~$t$ the freshest \emph{known} value is $V_{t-1}$. Relying on $V_{t-m-1}$ let us form $\nabla \tilde Z_t$ without forecasting the entire future constraint $g_t(x_{t-m}^t)$, and yet it does not affect the bound. Secondly, due to Assumption~\ref{ass:linearity}, the gradient of $f_{t+i}^{\,i}$ is the constant coefficient vector (independent of~$x_t$), and for the constraint
\(
  g_{t+i}^{i,+}(x)=\max\{0,\,a_{t+i}^{i}x+b_{t+i}^{i}\},
\)
we only need to predict the sign of $a_{t+i}^{i}x_t+b_{t+i}^{i}$. Thus the predictor only guesses the half‑space of~$x_t$, a far weaker requirement than guessing the exact $\tilde x_t$. However, this assumption can be lifted if a predictor is available that directly provides an estimate $\tilde x_t$ of $x_{t}$.

Finally, there is an interesting trade-off between the assumptions about the predictions and the problem structure. In general, satisfying the memory-based constraints $g_t(x_{t-m}^t)$ would require not only forecasting the future loss and constraints, but also their dependence on the yet-unknown decisions $x_{t+1},\ldots, x_{t+m}$, which in turn demands perfect predictions for a look‑ahead horizon $H=\Theta(\log T)$ as in \cite{yu2020power}. For shorter or imperfect forecasts, the sublinear bounds are not guaranteed. To sidestep this, we invoke Assumption \ref{ass:separable} and replace the comparator set $\mathcal X_T^m$ with $\mathcal X_T^{\mathrm{mp}}$. This relaxation allows us recasting it as a memory-less problem with {delayed gradients}. Consequently, we recover sublinear bounds on both regret an CCV, even under untrusted predictions. The reader might recall that similar concessions about the benchmark set are made in traditional COCO, where $\c X$ is reduced to $\c X_T$ so as to avoid the impossibility result of \cite{mannor2009online}. Making bolder assumptions for the availability of more informative predictions to learn against an expanded benchmark set is certainly a direction where our framework can be extended.

\section{Conclusions}

As discussed, \texttt{COCO-M\textsuperscript{2}} appears in many real systems, e.g., smart-grid energy budgets, battery-health limits, etc., and directly captures the handling of constraints in NSC. Our penalty method tightens the only prior \texttt{COCO-M\textsuperscript{2}} rates of \cite{liu2023online} ($\mathcal O(T^{2/3}\log^{2}T)$) to  
\(
\mathcal R_T=\mathcal O\!\bigl(m^{3/2}\sqrt{T\log T}\bigr) \)  and
\(\mathcal V_T=\mathcal O\!\bigl(\max\{T^{3/4},\,m^{3/2}\sqrt{T\log T}\}\bigr),
\)
and extends the analysis to the \texttt{COCO-M} case. Moreover, this is the first work to study \emph{untrusted} gradient forecasts for time-varying COCO problems with memory. The proposed optimistic algorithm achieves  
\(
\mathcal R_T\!=\!\mathcal O\!\bigl(\sqrt{\mathcal E_{Tf}}\bigr)\) and \(
\mathcal V_T\!=\!\mathcal O\!\bigl(\max\{\sqrt{\mathcal E_{Tg}}\log T,\,m\log T\}\bigr),
\)
matching the results of \cite{lekeufack2024optimistic} without memory and reducing to them when $m=0$. Indeed, previous COCO, OCO-M and optimistic OCO bounds emerge as special cases of our framework.  
Finally, as future work, these techniques can be extended to dynamic (adaptive) regret metrics via static-to-dynamic reductions, and with the design of penalties that react to time-varying windows and constraint hardness.

\newpage

\section*{Acknowledgments}
The work was supported in part by the Dutch National Growth Fund through the 6G flagship project “Future Network Services” and by the European Commission under Grants 101139270 (ORIGAMI) and 101192462 (FLECON-6G), and in part by the French government through the France 2030 program within the Celtic RAI-6green project.

\bibliography{aaai2026}


\clearpage
\newpage
\setcounter{secnumdepth}{2}
\appendix                    
\setcounter{section}{0}      
\renewcommand{\thesection}{Appendix~\Alph{section}}

\renewcommand{\thesubsection}{\Alph{section}.\arabic{subsection}}

In the following, we present complete proofs for the technical results in the main submission and include a dedicated numerical-experiments section. \ref{app:numerics} presents numerical experiments that corroborate our theoretical bounds, while \ref{app:section4} and \ref{app:section5} contain the full proofs of the theorems stated in Section \ref{sec:COCO with memory without predic} and Section \ref{sec:COCO with memory and prediction}, respectively, together with the auxiliary lemmas on which those proofs rely. {In \ref{app:section4}, we additionally state and prove bounds for special cases of the \texttt{COCO-M} problem.} Moreover, Table \ref{tab:related-work} presents a detailed state-of-the-art comparison: for every problem variant to which our algorithm applies, it lists our regret and constraint-violation bounds and places them side by side with the best bounds reported in the most relevant prior work. Finally, Figure \ref{fig:Notations} lists every symbol and domain used in the paper.

\begin{figure}[t]
\begin{center}
\includegraphics[width=0.99\columnwidth]{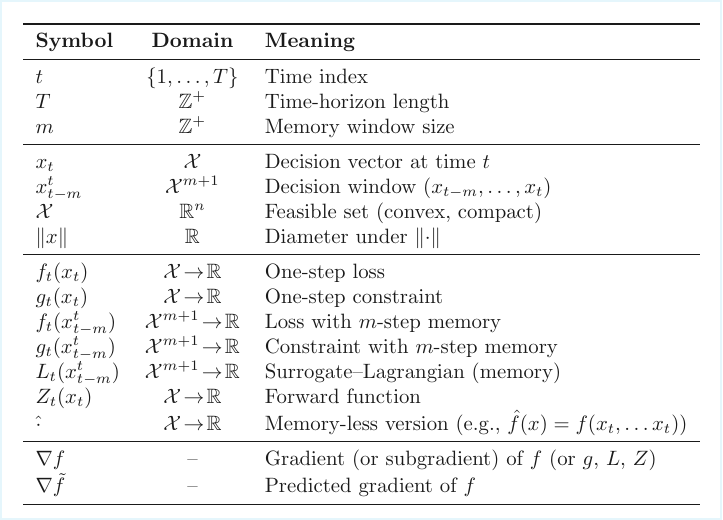}
\end{center}
\vspace{-4mm}
\caption{Notation used throughout the paper}
\label{fig:Notations}
\end{figure}

\begin{figure*}[t]
  \centering
  \begin{minipage}{0.45\textwidth}
    \centering
    \includegraphics[width=\linewidth]{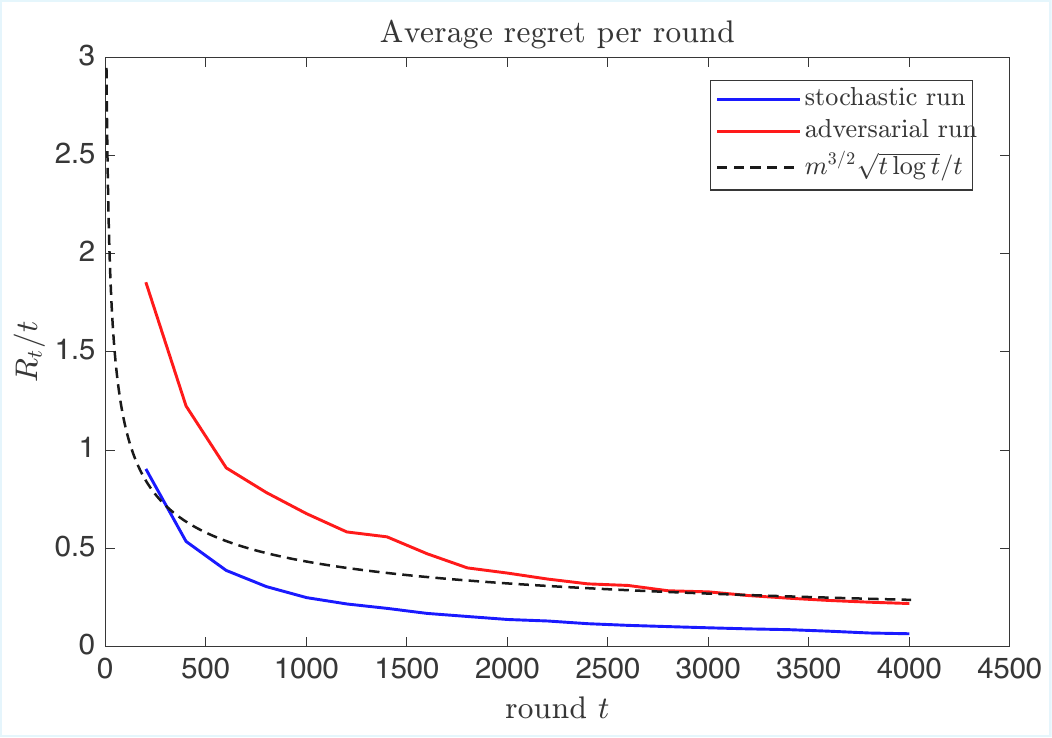}
    \vspace{-0.4\baselineskip}
    \par\smallskip
    \textbf{(a)} Average regret $R_t/t$
  \end{minipage}\hfill
  \begin{minipage}{0.45\textwidth}
    \centering
    \includegraphics[width=\linewidth]{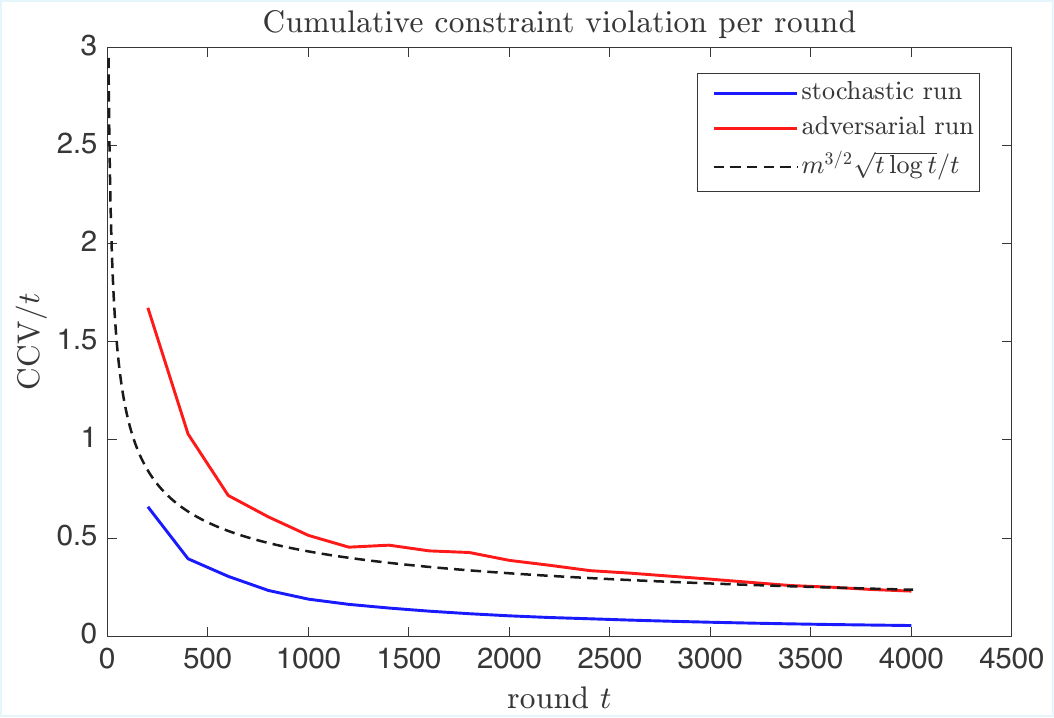}
    \vspace{-0.4\baselineskip}
    \par\smallskip
    \textbf{(b)} Cumulative constraint violation $V_t/t$
  \end{minipage}
  \caption{Performance of Algorithm \ref{algo:memory-coco} on
           $\texttt{COCO‑M}^{2}$ with $x^{\star}\!\in\!\mathcal{X}^{\text{m}}_T$.   }
  \label{fig:sim-plots}
\end{figure*}


\begin{table*}[t]
\centering
\footnotesize
\setlength{\tabcolsep}{5.5pt}
\begin{tabular}{|l|l|l|cc|c|p{3cm}|p{3.7cm}|}
\hline
\textbf{Reference} & \textbf{Constraints} & \textbf{loss} &
\multicolumn{2}{c|}{\textbf{Memory}} &
\textbf{Pred.} & \textbf{Regret} & \textbf{Violation} \\ \cline{4-5}
 & & & Obj. & Constr. & & & \\ \hline
Liu et al 
& adversarial        & convex &
\cmark & \cmark & \xmark &
$\mathcal{O}\!\bigl(T^{2/3}\log^{2}\!T\bigr)$  &
$\mathcal{O}\!\bigl(T^{2/3}\log^{2}\!T\bigr)$ (CCV) \\ \hline
Li et al 
& affine (static)     & convex &
\cmark & \cmark & \xmark &
$\widetilde{\mathcal{O}}\!\bigl(C_s\sqrt{T}\bigr)$ &
$\mathcal{O}(1)$ \\ \hline
Anava et al 
& none               & convex &
\cmark & \xmark & \xmark &
$\mathcal{O}\!\bigl(m^{3/2}\sqrt{T}\bigr)$ & n/a \\ \hline
Mhaisen et al 
& none           & linear &
\cmark & \xmark & \cmark &
$\mathcal{O}(\mathcal{E}_T)$  & n/a \\ \hline
Lekeufack et al 
& adversarial & convex &
\xmark & \xmark & \cmark &
$\mathcal{O}\!\bigl(\sqrt{\bar{\mathcal{E}}_f}\bigr)$  &
 $\mathcal{O}\!\bigl(\log T \, (\sqrt{\bar{\mathcal{E}}_g}+1)\bigr)$ CCV \\ \hline
\rowcolor{gray!30}
\textbf{This work}                      & 
\begin{tabular}[t]{@{}l@{}}%
adversarial\\[4pt]
adversarial\\[4pt]
adversarial \\[4pt]
adversarial\\[4pt]
ad. separable\\
\end{tabular}
       & \begin{tabular}[t]{@{}l@{}}%
convex\\[4pt]
convex\\[4pt]
convex \\[4pt]
ad. separable\\[4pt]
ad. separable\\
\end{tabular} &
\begin{tabular}[t]{@{}l@{}}%
\xmark\\[4pt]
\cmark\\[4pt]
\cmark \\[4pt]
\cmark\\[4pt]
\cmark\
\end{tabular} & \begin{tabular}[t]{@{}l@{}}%
\xmark\\[4pt]
\xmark\\[4pt]
\cmark \\[4pt]
\xmark\\[4pt]
\cmark\\
\end{tabular} &
\begin{tabular}[t]{@{}l@{}}%
\xmark\\[4pt]
\xmark\\[4pt]
\xmark \\[4pt]
\cmark\\[4pt]
\cmark
\end{tabular}  & 
\begin{tabular}[t]{@{}l@{}}%
$\mathcal{O}(\sqrt{T})$\\[2pt]
$\mathcal{O}(m^{\frac{3}{2}} \sqrt{T\log(T)})$\\[2pt]
$\mathcal{O}(m^{\frac{3}{2}}\sqrt{T\log T})$\\[2pt]
$\mathcal{O}(\sqrt{\mathcal{E}_T(f)})$\\[2pt]
$\mathcal{O}(\sqrt{\mathcal{E}_T(f)})$
\end{tabular}
& \begin{tabular}[t]{@{}l@{}}%
$\mathcal{O}(\sqrt{T\log T})$\\[2pt]
$\mathcal{O}\left(T^{3/4}\right)$\\[2pt]
$\mathcal{O}(m^{\frac{3}{2}}\sqrt{T\log T} \vee T^{\frac{3}{4}})$ \\[2pt]
$\mathcal{O}((\sqrt{\bar{\mathcal{E}}_T(g^+)}+1) \log T)$ \\[2pt]
$\mathcal{O}((\sqrt{\mathcal{E}_T(g^+)}+1) \log T)$ 
\end{tabular}
\\ \hline
\end{tabular}
\caption{Summary of the most relevant OCO/COCO results. Violation bounds are stated as \emph{CCV} ($\sum_{t} g_t^+$). Our work is the first to unify \emph{memory in both objective and constraint} with \emph{imperfect predictions} while guaranteeing simultaneous sub-linear regret and CCV. 
}
\label{tab:related-work}
\end{table*}

\section{Numerical evaluation}\label{app:numerics}

\subsection{Numerical Experiments}

To verify the theoretical bounds, we conduct numerical simulations on a small \texttt{COCO-M\textsuperscript{2}} problem involving quadratic losses and linear constraints, where the penalty function is \(\Phi(V) = \lambda V^{2}\). The experiments consider a one–dimensional case (\(d=1\)), where at each round \(t \in \{m,\dots,T\}\) the learner selects \(x_t \in \mathbb{R}\) and incurs loss:
\[
f_t(x_{t-m}^{t}) = \frac{1}{m+1} \sum_{i=0}^{m} \frac12 (x_{t-i} - c_t)^{2},
\]
and has a constraint:
\[
g_t(x_{t-m}^{t}) = \frac{1}{m+1} \sum_{i=0}^{m} d_t x_{t-i} - \delta,
\]
where \(c_t,d_t \in [-B,B]\) and \(\delta > 0\) ensures the bounded–gradient assumption required by our theory. The learner must keep the running average of \(g_t^+(x_{t-m}^{t})\) small while remaining within the Euclidean ball \(|x_t| \le R\).  

We evaluate two scenarios. In the stochastic setting, each coefficient pair \((c_t,d_t)\) is drawn i.i.d.\ from the centered uniform distribution \(\mathcal{U}[-\sigma,\sigma]\). In the adversarial setting, after observing \(x_t\), an oblivious adversary selects \((c_t,d_t)\) by drawing from \(\mathcal{U}[-\sigma,\sigma]\) with probability \(0.4\) and from \(\mathcal{N}(0,\sigma^2)\) with probability \(0.6\).  

Figure~\ref{fig:sim-plots}(a) reports the cumulative regret for different time-horizons of length $t$:
\[
\sum_{\tau=1}^{t} f_{\tau}(x_{\tau-m},\dots,x_{\tau}) - \min_{g_{\tau}(x) \le 0} f_{\tau}(x),
\]
and similarly Figure~\ref{fig:sim-plots}(b) shows the cumulative constraint violation
\[
\sum_{\tau=1}^{t} g_{\tau}^+(x_{\tau-m},\dots,x_{\tau}).
\]
Unless otherwise specified, the parameters are set to \(m = 3,\ t\leq T = 4000,\ R = 15,\ \sigma = 10,\ \delta = 1,\ B = \gamma\sigma\) with \(\gamma = 3,\) and \(\lambda = 1/\sqrt{t}\). 

{All code and data used for the computational experiment will be made publicly available at:
https://github.com/mohammadabduallah/COCO-Memory-AAAI26
}

\section{ Proofs of Section \ref{sec:COCO with memory without predic}}\label{app:section4}
\subsection{Proof of Theorems \ref{thm:full-memory}  and \ref{thm:obj-memory}}\label{appen:theom-obj-full-mem}

The proof proceeds in two stages. First, we derive the regret and CCV bounds for the \emph{memory-less} version of the problem; and next, we lift those bounds to the memory-dependent setting via the Lipschitz property of the function, as shown in \eqref{eq:objective_memory_decom} and \eqref{eq:constraint_memory_decom}.

\paragraph{Memory-less problem }
\begin{proof}[Proof Theorem \ref{thm:obj-memory}]
We begin by bounding the regret \(\mathcal{R}_{T}(\hat{\mathcal{L}})\).
This will be fed directly to the bounds for
\(\hat{\mathcal{R}}_{T}^{mc}\) and \(\hat{V}_{T}\).
Note that \(\hat{V}_{T}\) coincides with
\(\mathcal{V}_{T}^{c}\) in \(\texttt{COCO\text{-}M}\), whereas in \(\texttt{COCO\text{-}M}^{2}\) it becomes
\(\hat{\mathcal{V}}_{T}^{mc}\), i.e., the CCV associated with the memory-less constraint. 

By Lemma~\ref{lem:OGD_adaptive_bound}, and using that $\Phi(V)=\lambda V^2$ we have:
\begin{align*}
    \c R_T(\hat{\c L}) &\leq \sqrt{2}\| \mathcal{X}\| \sqrt{\sum_{t=1}^{T} \| \nabla \hat{\c L}_t(x_t)\|^2}\\
    &\stackrel{(i)}\leq  \sqrt{2}\| \mathcal{X}\| \sqrt{T(L_f+ 2\lambda \hat{V}_T L_g)^2}\\
    &\leq  \sqrt{2}\| \mathcal{X}\|\sqrt{T}(L_f+2\lambda \hat{V}_TL_g)\\
    &\leq \sqrt{2 T} \|\mathcal{X}\| L_f + 2 \lambda \hat{V}_T (\sqrt{2T} \|\mathcal{X}\|) L_g,
\end{align*}
where in $(i)$ we used the Lipschitz condition in
Assump.~\ref{ass:lipschitz}: 
\begin{equation}\label{eq:1}
    \sup_{t}\bigl\|\nabla L_t(x)\bigr\|\;\le\; L_f + \Phi'(\hat{V}_T) L_g=L_f+ 2\lambda \hat{V}_T L_g.
\end{equation}

Using $2ab \leq a^2 + b^2$ with $a \!=\! \sqrt{\lambda }\hat{V}_T$, $b\!=\!\sqrt{2 \lambda T} \|\mathcal X\| L_g$ we can upper bound the LHS of equation \eqref{eq:1} 
\[
\mathcal{R}_T(\hat{\mathcal{L}}) \leq \sqrt{2 T} \|\mathcal{X}\| L_f + \lambda \hat{V}_T^2 + 2\lambda T \|\mathcal X\| ^2 L_g^2.
\]

Now by the regret decomposition Lemma \ref{lem:Regret_decom_no_predic} we have that:
\[
\Phi(\hat{V}_T) - \Phi(\hat{V}_m) + \hat{\mathcal{R}}_T^{mc}\leq \mathcal{R}_T(\hat{\mathcal{L}}).
\]

By setting $\lambda=1/\sqrt{T}$ we get the following:
\begin{align*}
\hat{\mathcal{R}}_T^{mc} &\leq \sqrt{2 T} \|\mathcal{X}\| L_f + \lambda \hat{V}_T^2  -\lambda \hat{V}_T^2 \\
&+ 2\lambda T \| \mathcal X\| ^2 L_g^2  +\Phi(\hat{V}_m)\\
&\leq \sqrt{2 T} \|\mathcal{X}\| L_f + 2 \sqrt{T} \|\mathcal X\| ^2 L_g^2 +\Phi(\hat{V}_m)   \:\: \text{($\lambda =\frac{1}{\sqrt{T}}$)}\\ &=\mathcal{O}(\sqrt{T}).
\end{align*}

For $\hat{V}_T$, by the regret decomposition Lemma \ref{lem:Regret_decom_no_predic} we have:
\[
\Phi(\hat{V}_T) -\Phi(\hat{V}_m) \leq- \hat{\mathcal{R}}_T^{mc} + \mathcal{R}_T(\hat{\mathcal{L}}),
\]
and knowing that $\hat{\mathcal{R}}_T^{mc}\geq - 2FT$, we get:
\begin{align*}
      \lambda \hat{V}_T^2 - \lambda \hat{V}_m^2 & \leq \sqrt{2T}\|\mathcal{X}\| L_f+ 2\lambda \hat{V}_T(\sqrt{2T}\|\mathcal{X}\|L_g) + 2FT.
\end{align*}
Dividing both sides by $\lambda $:
$$\hat{V}_T^2- 2\hat{V}_T(\sqrt{2T}\|\mathcal{X}\| L_g)  \leq T\sqrt{2}\|\mathcal{X}\| L_f + 2FT^{3/2} +\hat{V}^2_m$$
we set $\hat{V}_m\!=\!0$ w.l.o.g., and add $2T\|\mathcal{X}\|^2L_g^2$ to both sides: 
$$(\hat{V}_T -\sqrt{2T}\|\mathcal{X}\| L_g)^2 \leq 2T\|\mathcal{X}\|^2L_g^2+ T \sqrt{2}\|\mathcal{X}\|L_f + 2FT^{3/2}.$$
And finally, we have that:
\begin{align*}
\hat{V}_T &\leq \sqrt{2T\|\mathcal{X}\|^2L_g^2+ T \sqrt{2}\|\mathcal{X}\|L_f + 2FT^{3/2}} \\
&+\sqrt{2T}\|\mathcal{X}\| L_g = \mathcal{O}(T^{3/4}) \\
\end{align*}

\paragraph{Lifting bounds to $\texttt{COCO-M}$ problem.} 
We note that in this problem it holds $\hat{V}_T=\mathcal{V}_T^c$; thus the bound we already have for $\hat{V}_T$  applies directly here. Hence it is
$\mathcal{V}_T^c\!=\!\mathcal{O}(T^{\frac{3}{4}})$.

For the regret we have the following relation:
\begin{align*}
	\c{R}_{T}^{c}
	&= \sum_{t=m}^{T}
		f_t(x_{t-m},\dots,x_t)-f_t(x_t,\dots,x_t)+\hat{\c{R}}_{T}^c\\
        &\leq \sum_{t=m}^{T}| f_t(x_{t-m},\dots,x_{t})-f_t(x_{t},\dots,x_t) |+ \hat{\c{R}}_{T}^c.
\end{align*}
Building on \cite{anava2015online} and using the standard one–step OGD bound
\(
\|x_{t+1}-x_t\|\le\eta_t
     \bigl\|\nabla\hat{\mathcal{L}}_t(x_t)\bigr\|,
\)
we obtain:
\begin{align*}
   &| f_t(x_{t-m}, \dots, x_t) - f_t(x_t, \dots, x_t) |^2 \\
   &\le L_f^2 \| (x_{t-m}, \dots, x_t) - (x_t, \dots, x_t) \|^2 \\
    &\leq L_f^2 \sum_{j=1}^{m} \| x_t - x_{t-j} \|^2 \\
    &\le L_f^2 \sum_{j=1}^{m} \left( \sum_{l=1}^{j} \| x_{t-l+1} - x_{t-l} \| \right)^2 \\
    &\leq L_f^2 \sum_{j=1}^{m} \left( \sum_{l=1}^{j} \eta_{t-l}\| \nabla \hat{\mathcal{L}}_{t-l}(x_{t-l} \| \right)^2
\end{align*}
where $\eta_t$ is the step-size of the OGD algorithm: 
\[
\eta_t = \frac{ \|\mathcal{X}\|}{\sqrt{2}\sqrt{\sum_{k=1}^t \|\nabla \hat{\mathcal{L}}_k(x_k)\|^2}}.
\]
Thus, we can write $\c{R}_{T}^{c} \leq$
\begin{align*}
\hat{\c{R}}_{T}^c +\frac{\|\mathcal{X}\|}{\sqrt{2}} L_f \sqrt{  \sum_{j=1}^{m} \left( \sum_{l=1}^{j}\frac{\| \nabla \hat{\mathcal{L}}_{t-l}(x_{t-l} \|}{\sqrt{\sum_{\tau=1}^{t-l} \|\nabla \hat{\mathcal{L}}_\tau(x_\tau)\|^2}} \right)^2}
\end{align*}

Define $\delta_t = \sum_{k=1}^t \|\nabla \hat{\mathcal{L}}_k(x_k)\|^2$, then 
\[
\frac{\| \nabla \hat{\mathcal{L}}_{t-l}(x_{t-l})\|}{\sqrt{\delta_{t-l}}}=\sqrt{\frac{\delta_{t-l}-\delta_{t-l-1}}{\delta_{t-l}}} \in (0,1].
\]
Let $v_j\!\doteq\!\sum_{l=1}^{j}\sqrt{\frac{\delta_{t-l}-\delta_{t-l-1}}{\delta_{t-l}}} $, then by Cauchy–Schwarz:
\begin{equation*}
    v_j^2 \leq j \sum_{l=1}^{j}\frac{\delta_{t-l}-\delta_{t-l-1}}{\delta_{t-l}}  \leq m\sum_{l=1}^{m}\frac{\delta_{t-l}-\delta_{t-l-1}}{\delta_{t-l}}.
\end{equation*}
Taking the summation from $j=1$ to $m$, we have:
\[
\sum_{j=1}^m v_j^2 \leq m^2 \sum_{\ell=1}^m \frac{\delta_{t-l} - \delta_{t-l-1}}{\delta_{t-l}}.
\]
This leads to the following: 
\begin{align*}
   & L_f \sum_{t=m}^T \sqrt{ \sum_{j=1}^m \left( \sum_{l=1}^j \eta_{t-l} \|\nabla \hat{\mathcal{L}}_\ell(x_{t-l})\| \right)^2 }\\
   &\leq m L_f \frac{\|\mathcal{X}\|}{\sqrt{2}} \sum_{t=m}^{T} \sqrt{\sum_{j=1}^{m}\frac{\delta_{t-l}-\delta_{t-l-1}}{\delta_{t-l}}}.
\end{align*}
Now by the concavity of the square root, we can write: 
\begin{align}
    \frac{1}{T-m} \sum_{t=m}^{T}\sqrt{y_t}\leq \sqrt{\frac{\sum_{t=m}^{T}y_t}{T-m}}
\end{align}
and substituting $y_t\!=\!\sum_{l=1}^{m}\frac{\delta_{t-l}-\delta_{t-l-1}}{\delta_{t-l}}$, we get:
\begin{align*}
&m L_f \sqrt{2}\|\mathcal{X}\| \sum_{t=m}^{T} \sqrt{\sum_{j=1}^{m}\frac{\delta_{t-l}-\delta_{t-l-1}}{\delta_{t-l}}} \\
&\stackrel{(ii)}\leq m L_f  \frac{\|\mathcal{X}\|}{\sqrt{2}}  \sqrt{T-m} \sqrt{\sum_{t=m}^{T}\sum_{l=1}^{m} \frac{\delta_{t-l}-\delta_{t-l-1}}{\delta_{t-l}}}\\
&\leq m L_f  \frac{\|\mathcal{X}\|}{\sqrt{2}} \sqrt{T} \sqrt{m\sum_{t=1}^{T}\frac{\delta_t-\delta_{t-1}}{\delta_t}} \\
&\leq m^{\frac{3}{2}} L_f  \frac{\|\mathcal{X}\|}{\sqrt{2}} \sqrt{T}\sqrt{\sum_{t=1}^{T}\frac{\delta_t-\delta_{t-1}}{\delta_t}},
\end{align*}
where $(ii)$ holds because the double sum lists every element $m$-times for each $t$.

We have that for any $0<b<a$, $\frac{a-b}{a}\leq \log a -\log b $, and without loss of generality we assume that $\delta_0=1$, thus we get the following:
\begin{align*}
\sum_{t=1}^{T}\frac{\delta_t-\delta_{t-1}}{\delta_t} &\leq \sum_{t=1}^{T} \log(\delta_t) -\log (\delta_{t-1})\\
&\leq \log(\delta_T)-\log(\delta_0)\leq \log(\delta_T)
\end{align*}
Thus we have that: 
\begin{align*}
    &m L_f  \frac{\|\mathcal{X}\|}{\sqrt{2}}  \sum_{t=m}^{T} \sqrt{\sum_{j=1}^{m}\frac{\delta_{t-l}-\delta_{t-l-1}}{\delta_{t-l}}} \\
    & \leq m^{\frac{3}{2}} L_f  \frac{\|\mathcal{X}\|}{\sqrt{2}}   \sqrt{T\log(\delta_T)}\\
    &\leq m^{\frac{3}{2}} L_f  \frac{\|\mathcal{X}\|}{\sqrt{2}}   \sqrt{T}\sqrt{\log\Big(\sum_{t=1}^{T}\|\nabla \hat{\mathcal{L}}_t(x_t)\|^2}\Big)\\
  &  \leq m^{\frac{3}{2}} L_f  \frac{\|\mathcal{X}\|}{\sqrt{2}}   \sqrt{T}\sqrt{\log\Big(T(L_f+\Phi'(V_T)L_g )^2\Big) }\\
&  \leq  m^{\frac{3}{2}} L_f  \frac{\|\mathcal{X}\|}{\sqrt{2}}  \sqrt{T} \sqrt{\log(T) + 2\log(L_f+\sqrt{T}L_g)}
\end{align*}
As a result we have:
\begin{align*}
    \mathcal{R}_T^{mc}&\leq \sqrt{2 T} \|\mathcal{X}\| L_f + 2 \sqrt{T} \| X\| ^2 L_g^2  \\
    &+ m^{\frac{3}{2}} L_f  \frac{\|\mathcal{X}\|}{\sqrt{2}}  \sqrt{T} \sqrt{\log(T) + 2\log(L_f+\sqrt{T}L_g)} \\
    &= \mathcal{O}\Big(m^{\frac{3}{2}}\sqrt{T\log T}\Big)
\end{align*}

{Since we also have $$\mathcal{V}_T^c=\hat{V}_T=\mathcal{O}(T^{\frac{3}{4}}),$$ this completes the proof of Theorem \ref{thm:obj-memory}.}
\end{proof}

\paragraph{Bounds for $\texttt{COCO-M}^2$}
\begin{proof}[Proof of Theorem \ref{thm:full-memory}]
Here the constraint is memory–dependent,
\(g_t(x_{t-m}^{\,t})\).
Because the regret analysis involves only the use of $\hat{\mathcal{V}}_T$, that is in this case $\hat{\mathcal{V}}_T^{mc}$ the CCV for the memory-less constraint $\hat{g}_t$, the bound is identical to that obtained for \(\texttt{COCO-M}\):
\[
\mathcal{R}_{T}^{mc}
  \;=\;
  \mathcal{O}\!\bigl(m^{3/2}\sqrt{T\log T}\bigr).
\]

Now for $\mathcal{V}_T^{mc}$ we use the same decomposition as before. Since \(g_t\) is \(L_g\)-Lipschitz, its positive part
\(g_t^{+}\) inherits the same constant. Hence we can write:
\begin{align*}
\mathcal{V}_{T}^{mc}
  &\leq 
     \sum_{t=m}^{T}
     | g_t^+(x_{t-m},\dots,x_{t})-g_t^+(x_t,\dots,x_t)|
     +
     \hat{\mathcal{V}}_{T}^{\,c} 
\end{align*}
where
\(\hat{\mathcal{V}}_{T}^{\,c}=\hat{V}_T = \mathcal{O}(T^{\frac{3}{4}})\) ( already bounded before).

Following the same steps as before, but with \(L_g\) in place of \(L_f\), we obtain:
\begin{align}
    & \sum_{t=m}^{T}
     | g_t^+(x_{t-m},\dots,x_{t})-g_t^+(x_t,\dots,x_t)|\\
     &\leq m^{\frac{3}{2}} L_g  \frac{\|\mathcal{X}\|}{\sqrt{2}}  \sqrt{T} \sqrt{\log(T) + 2\log(L_f+\sqrt{T}L_g)}.
\end{align}
Combining both results, yields:
\begin{align*}
    \mathcal{V}_T^{mc}&\leq  \sqrt{2T}\|\mathcal{X}\|\\
    &+\sqrt{2T\|\mathcal{X}\|^2L_g^2+ T \sqrt{2}\|\mathcal{X}\|L_f + 2FT^{3/2}} \\
    &+ m^{\frac{3}{2}} L_g  \frac{\|\mathcal{X}\|}{\sqrt{2}}  \sqrt{T} \sqrt{\log(T) + 2\log(L_f+\sqrt{T}L_g)}\\
    &= \mathcal{O}\left(\max\{T^{3/4},m^{\frac{3}{2}}\sqrt{T\log T}\right)
\end{align*}
which completes the proof.
\end{proof}

\subsection{\(\texttt{COCO-M}\) with a Short Memory Window}\label{sec:short-memory}
When the memory length satisfies
\[
  m\le
  T^{1/6}/({\log T})^{1/3},
\]
we can obtain tighter bounds by replacing the quadratic penalty with the \emph{exponential} penalty
\(
  \Phi(V)=e^{\lambda V}-1.
\)
The next theorem gives the resulting rates.

\begin{theorem}[Regret and CCV with memory-free constraint]\label{thm:short-window}
When \textnormal{(i)}~Assumptions \ref{ass:domain}, \ref{ass:convex}, and \ref{ass:lipschitz} hold, \textnormal{(ii)} the constraint is memory-less, $g_t(x_t)$, and \textnormal{(iii)} we use the step 
\(\eta_t= \frac{\sqrt{2}\|\mathcal{X}\|}{2\sqrt{\sum_{\tau=1}^{t} \|\nabla \hat{\mathcal{L}}_t(x_t)\|^2}}\)
and the penalty \(\Phi(V)\!=\!\exp(\lambda V)\!-\!1\)
\begin{align*}
\text{with} \ \ \lambda= \frac{0.5}{\sqrt{2T}\|\mathcal{X}\|L_g  +  m^{\frac{3}{2}}  \|\mathcal{X}\| \sqrt{TL_fL_g} }
\end{align*}
Then, for any $T\!\ge\! m$, it holds:
\begin{align}
\c{R}_{T}^{mc}&= \mathcal{O}\Big(\sqrt{T}+m^{\frac{3}{2}} \sqrt{T\log(T)}\Big),\\ 
\c{V}_T^{c}&= \mathcal{O}\Big(\sqrt{T\log T}+m^{\frac{3}{2}}\sqrt{T}\log T\Big)\notag
\end{align}
\end{theorem}

\begin{proof}
We first bound the \emph{memory-less} problem and then lift the result to the memory–dependent setting via the Lipschitz property of \(f_t\). The only change, compared to our previous analysis, is that we now work with the \emph{exponential} penalty. By the regret–decomposition lemma, we can write:
\begin{equation*}
    \hat{\mathcal{R}}_T^{mc}\leq \sqrt{2 T} \|\mathcal{X}\| L_f + \Phi'(\hat{V}_T) \left(\sqrt{2T} \|\mathcal{X}\|\right) L_g -\Phi(\hat{V}_T).
\end{equation*}

As before, we define $\delta_t\!\doteq\!\sum_{t=1}^t \|\nabla \hat{\mathcal{L}}_t(x_t)\|^2$ to bound the memory dependent term, and we have:
\begin{align*}
	\c{R}_{T}^{mc} &\leq  \hat{\c{R}}_{T}^{mc}+  \sum_{t=m}^{T}| f_t(x_{t-m},\dots,x_{t})-f_t(x_{t},\dots,x_t) |\\
   &\leq  \sqrt{2 T} \|\mathcal{X}\| L_f + \Phi'(\hat{V}_T) (\sqrt{2T} \|\mathcal{X}\|) L_g -\Phi(\hat{V}_T)\\
    & + m^{\frac{3}{2}} L_f  \frac{\|\mathcal{X}\|}{\sqrt{2}}   \sqrt{T\log(\delta_T)}.
\end{align*}

It holds:
\begin{align*}
     & \sqrt{T\log \delta_T} 
    \leq \sqrt{T} \sqrt{\log \left(T \left(L_f+ L_g \Phi'(\hat{V}_T)\right)^2\right) }\\
   & \leq \sqrt{T}\sqrt{\log T+ 2\log L_f + 2 \log\Big(1+\frac{L_g}{L_f} \Phi'(\hat{V}_T)\Big)}\\
  &\stackrel{(i)}\leq   \sqrt{T} \left(\sqrt{\log T} + \sqrt{2\log L_f} + \sqrt{2\frac{L_g}{L_f}\Phi'(\hat{V}_T)}\right) 
\end{align*}
where in $(i)$ we used that $\log(1+\alpha)\leq \alpha$. 

Thus, we can write:
\begin{align*}
    &\mathcal{R}_T^{mc}\leq \sqrt{2 T} \|\mathcal{X}\| L_f + \Phi'(\hat{V}_T) \left(\sqrt{2T} \|\mathcal{X}\|\right) L_g \\
    &\!+\!m^{\frac{3}{2}} L_f \frac{\|\mathcal{X}\|}{\sqrt{2}} \sqrt{T}\left(\!\sqrt{\log T}\!+\!\sqrt{2\log L_f} +\! \sqrt{2\frac{L_g}{L_f}}\Phi'(\hat{V}_T)\!\right) \\
    &-\Phi(\hat{V}_T)+\Phi(\hat{V}_m)\\
   & \leq \sqrt{2T}\|\mathcal{X}\|L_f +m^{\frac{3}{2}} L_f \frac{\|\mathcal{X}\|}{\sqrt{2}} \sqrt{T\log T} \\
   &+ \Phi(\hat{V}_m)+ m^{\frac{3}{2}} L_f \|\mathcal{X}\|\sqrt{\log L_f} + \Phi'(\hat{V}_T) \sqrt{2T}\|\mathcal{X}\|L_g\\
   &  + m^{\frac{3}{2}}  \|\mathcal{X}\| \sqrt{TL_fL_g}\Phi'(\hat{V}_T) -\Phi(\hat{V}_T)\\
   & \leq \sqrt{2T}\|\mathcal{X}\|L_f +m^{\frac{3}{2}} L_f \frac{\|\mathcal{X}\|}{\sqrt{2}} \sqrt{T\log T }  + \exp(\lambda \hat{V}_m)\\
   &+ m^{\frac{3}{2}} L_f \|\mathcal{X}\|\sqrt{\log L_f} \\
   &+\exp(\lambda \hat{V}_T) \Big(\lambda \sqrt{2T}\|\mathcal{X}\|L_g  + \lambda  m^{\frac{3}{2}}  \|\mathcal{X}\| \sqrt{TL_fL_g} -1\Big)
\end{align*}
Knowing that 
\[
\lambda= \frac{0.5}{\sqrt{2T}\|\mathcal{X}\|L_g  +  m^{\frac{3}{2}}  \|\mathcal{X}\| \sqrt{TL_fL_g} },
\]
we get:
\begin{align}\label{eq:small-m-regret}
\!\mathcal{R}_T^{mc}&\!\leq\!\sqrt{2T}\|\mathcal{X}\|L_f +m^{\frac{3}{2}} L_f \frac{\|\mathcal{X}\|}{\sqrt{2}} \sqrt{T\log T}\!+\!\exp(\lambda \hat{V}_m)\notag\\
   &+ m^{\frac{3}{2}} L_f \|\mathcal{X}\|\sqrt{\log L_f} \\
   &= \mathcal{O}\left(\sqrt{T}+ m^{3/2}\sqrt{T\log T}\right)
\end{align}

Now we bound the CCV. Recall that $\hat{V}_T$ coincides with $\mathcal{V}_T^{c}$. Invoking the regret-decomposition lemma \ref{lem:Regret_decom_no_predic}, and knowing that $\hat{\mathcal{R}}_T^{cm} \geq -2FT$, gives:
\begin{align*}
    \Phi(\hat{V}_T) &\leq \mathcal{R}_T(\hat{\mathcal{L}}) +2FT + \Phi(\hat{V}_m)\\
    &\leq  \sqrt{2T}\|\mathcal{X}\| L_f +\Phi'(\hat{V}_T) L_g\sqrt{2T}\|\mathcal{X}\| +2FT
\end{align*}

Thus we have: 
\begin{align*}
     \exp(\lambda \hat{V}_T)\left(1-\lambda L_g \|\mathcal{X}\|\sqrt{ 2T}\right)&\leq \|\mathcal{X}\| L_f \sqrt{2T} +2FT.
\end{align*}
And since $\lambda= \frac{0.5}{\sqrt{2T}\|\mathcal{X}\|L_g  +  m^{\frac{3}{2}}  \|\mathcal{X}\| \sqrt{TL_fL_g} } $, it is: $$\frac{1}{2}\leq \left(1-\lambda L_g \|\mathcal{X}\|\sqrt{2T}\right).$$
This gives:
\begin{align}\label{eq:small-m-ccv}
    \frac{1}{2}  \exp(\lambda \hat{V}_T)&\leq \|\mathcal{X}\| L_f \sqrt{2T} +2FT\Rightarrow\nonumber\\
  \hat{V}_T&\leq \frac{1}{\lambda} \log\Big(\|\mathcal{X}\| L_f \sqrt{2T} +2FT\Big)\nonumber\\
  &\leq 2 \Big(\sqrt{2T}\|\mathcal{X}\|L_g  +  m^{\frac{3}{2}}  \|\mathcal{X}\| \sqrt{TL_fL_g} \Big)\nonumber\\
  &\times \log\Big(\|\mathcal{X}\| L_f \sqrt{2T} +2FT\Big). 
\end{align}
And as mentioned before, it is $\mathcal{V}_T^c= \hat{V}_T$, thus $$\mathcal{V}_T^c= \mathcal{O}\Big(\sqrt{T}\log T + m^{\frac{3}{2}}\sqrt{T}\log(T)\Big)$$
This completes the proof.

\end{proof}

\textbf{Discussion.} The regret bound matches that of Theorem \ref{thm:obj-memory}, but the CCV changed to \(\mathcal{O}\bigl(m^{3/2}\sqrt{T}\log T\bigr)\), which is strictly
smaller than the \(\mathcal{O}(T^{3/4})\) rate of
Theorem~\ref{thm:full-memory} when
\(m\;\le\;
  T^{1/6}/({\log T})^{1/3}\).
If \(m\!=\!0\) the bounds reduce to
\(\mathcal{O}(\sqrt{T})\) for the regret and
\(\mathcal{O}(\sqrt{T}\log T)\) for CCV, recovering the optimal memory-free rates of \cite{sinha2024optimal}.
\section{ Proofs of Section \ref{sec:COCO with memory and prediction}}\label{app:section5}

We provide here the missing proofs for the results of this section, as well as details regarding the optimistic subroutine from \cite{flaspohler2021online} that is used in Algorithm \ref{algo:memory-coco_predic}. As well as presenting the results and outlining the modifications required to analyze optimistic $\texttt{COCO-M}$.

\subsection{Proof of Lemma \ref{lem:memory-prediction-forward-decomposition}}\label{app:memory-prediction-regret}
We first provide an updated regret decomposition lemma that couples the memory-based surrogate loss, $\mathcal{L}_t(x_{t-m}^{t})$, with the regret and CCV. This result updates the previous decomposition lemma in Sec. \ref{sec:COCO with memory without predic} that uses the memory-less surrogate functions.

\begin{align*}
\Phi\bigl(V_{T}\bigr) - \Phi\bigl(V_{m-1}\bigr)+\mathcal{R}_T^{mc}\leq G \,(m+1)\,\Phi'\bigl(V_{T}\bigr).
\end{align*}

\begin{lemma}[Regret decomposition]
\label{lem:regret-decomp_prediction}
Let $\Phi(\cdot)= \exp(\lambda \cdot)$ be the penalty function, then for any $T \ge m$ and any $u \in \mathcal{X}_T^{m}$ or $\mathcal{X}_T^{mp}$, yields:
\begin{align*}
  &  \Phi\bigl(V_{T}\bigr) - \Phi\bigl(V_{m-1}\bigr)
    +\mathcal{R}_T^{mc}\\
    &\leq
    \mathcal{R}_{T}(\mathcal{L})
    +
    G \,(m+1)\,\Phi'\bigl(V_{T}\bigr).
\end{align*}
\end{lemma}
\begin{proof}[Proof of Lemma \ref{lem:regret-decomp_prediction}]
First, by the convexity of $\Phi$, we can write:
\begin{align*}
\Phi(V_{t}) 
&\leq \Phi(V_{t-1}) + \Phi'(V_{t}) g_t^+(x_{t-m}, \dots, x_t).
\end{align*}
Adding on both sides $f_t(x_{t-m}^t) - f_t(u, \dots, u)$, we get:
\begin{align*}
&\Phi(V_{t}) - \Phi(V_{t-1}) + f_t(x_{t-m}^t) - f_t(u, \dots, u) \\
&\leq f_t(x_{t-m}^t) - f_t(u, \dots, u) 
+ \Phi'(V_{t}) g_t^+(x_{t-m}^t) \\
&\leq\Big[ f_t(x_{t-m}^t)+\Phi'(V_{t-m-1})g^+(x_{t-m}^t)\Big]\\
&- \Big[f_t(u, \dots, u) + \Phi'(V_{t-m-1}) g_t^+(u, \dots, u)\Big]\\
&+g_t^+(x_{t-m}, \dots, x_t) \Big[ \Phi'(V_{t}) -\Phi'(V_{t-m-1})\Big]
\end{align*}
where we used that $g_t^+(u,\dots,u)=0$ (since $u \in \mathcal{X}_T^{m}$). Telescoping from $t=m $ to $T$, we have:
\begin{align*}
&\Phi(V_{T})-\Phi(V_{m-1}) + \mathcal{R}_T^{mc}\\
&\leq \mathcal{R}_T(\mathcal{L})+\sum_{t=m}^{T}g_t^+(x_{t-m}^t)\Big[ \Phi'(V_{t}) -\Phi'(V_{t-m-1})\Big]\\
&\leq \mathcal{R}_T(\mathcal{L}) + G (m+1) \Phi'(V_{T})
\end{align*}

\end{proof}

Having established this lemma, we now prove Lemma \ref{lem:memory-prediction-forward-decomposition} by bounding $\mathcal{R}_T(\mathcal{L})$ with the regret $\mathcal{R}_T(Z)$ of the forward function.
\begin{proof}[Proof of Lemma \ref{lem:memory-prediction-forward-decomposition}]
\begin{align*}
&\sum_{t=m}^{T}\mathcal{L}_t\!\bigl(x_{t-m
}^t\bigr)=\sum_{t=m}^{T}f_t\!\bigl(x_{t-m}^{t}\bigr)+\Phi'\!\bigl(V_{t-m-1}\bigr)\,g_t^{+}\!\bigl(x_{t-m}^{t}\bigr) \\[2pt]
  &\stackrel{(i)}\le\sum_{t=m}^{T}\sum_{i=0}^{m}
        \Bigl[
          f_t^{i}(x_{t-i})
          +\Phi'\!\bigl(V_{t-m-1}\bigr)\,g_{t}^{+,i}(x_{t-i})
        \Bigr] \\[2pt]
  &=\sum_{t=m}^{T}\sum_{i=0}^{m}
        \Bigl[
          f_{t+i}^{i}(x_{t})
          +\Phi'\!\bigl(V_{t-m-1+i}\bigr)\,g_{t+i}^{+,i}(x_{t})
        \Bigr]\\
        &\stackrel{(ii)}=\sum_{t=m}^{T}Z_t(x_t),
\end{align*}
where in $(i)$ we used that $g_t^+\leq \sum_{i}g_t^{+,i}$, and in $(ii)$ we used the assumption that $f_t(\cdot)=g_t(\cdot)=0$, $\forall t\le m, t>T$. And because of the considered benchmark $ u \in \mathcal{X}_T^{mp}$, we have
\[
\sum_{t=m}^{T}Z_t(u)=
\sum_{t=m}^{T}\mathcal{L}_t(u,\dots,u).
\]
Replacing this result above, concludes the proof. 
\end{proof}

\subsection{ODAFTRL update rule}\label{app:ODAFTRL-detialed}

At the end of round~$t$ the learner hands ODAFTRL three objects:

\begin{itemize}
\item the cumulative \emph{revealed} gradients $\nabla Z_{1:t-m}$;
\item the \emph{hint} $h_{t+1}$, which fills in all delayed and still-unseen gradients for round $t+1$ of the forward function $Z$;
\item the most recent prediction error $\epsilon_{t-m}$.
\end{itemize}

Given these inputs the algorithm chooses
\begin{align}
x_{t+1}&\leftarrow\texttt{ODAF}\!\bigl(\nabla Z_{1:t-m},h_{t+1},\epsilon_{t-m}\bigr)\\
      &=\arg\min_{x\in\mathcal X}
      \Bigl\langle \nabla Z_{1:t-m}+h_{t+1},\,x\Bigr\rangle
      +\mu_{t+1}\,r(x). \nonumber
\end{align}
where $r$ is a nonnegative 1-strongly convex function,  that is selected according to the geometry of $\mathcal X$. 
\noindent
\emph{Prediction-error–adaptive regulariser.}
ODAFTRL calibrates the weight
\[
\mu_{t+1}=
\frac{2}{\alpha}\max_{j\le t-m-1}a_{j-m+1:j}
\;+\;
\frac{1}{\alpha}\sqrt{\sum_{i=1}^{t-m}\bigl(a_i^{2}+2\alpha b_i\bigr)},
\]
with
\begin{align*}
b_{t}
&=
\mathrm{huber}\Bigl(\bigl\lVert h_{t}\;-\;\nabla Z_{t-m:t}\bigr\rVert,\,\lVert \nabla Z_{t}\rVert\Bigr),
\\
a_{t}
&=
\|\mathcal{X}\|\,\min\Bigl(\bigl\lVert h_{t}\;-\;\nabla Z_{t-m:t}\bigr\rVert,\,\lVert \nabla Z_{t}\rVert\Bigr), 
\end{align*}
and $\mathrm{huber}(x,y)=\tfrac12x^{2}-\tfrac12(\lvert x\rvert-\lvert y\rvert)_{+}^{2}$.
When $h_t$ is accurate ($h_t\approx \nabla Z_{t-m:t}$) we have
$a_i,b_i\!\downarrow0$ and therefore $\mu_{t+1}\!\downarrow$, so the
update is \emph{aggressive}, dominated by the linear loss and giving
more weight to predictions.
Conversely, large errors inflate $a_i,b_i$, hence
$\mu_{t+1}\!\uparrow$; the strongly-convex term then tempers the step,
yielding a \emph{conservative} move that avoids over-reacting to unreliable hints. All functional details and parameter choices follow \cite{flaspohler2021online}, and we note that one can also use as subroutine in Algorithm 2 \ref{algo:memory-coco_predic} another optimistic-delay algorithm variant from that work.

\subsection{Proof of Theorem \ref{Theom:memory-prediction}}\label{app:memory-prediction}

Before proving Theorem \ref{Theom:memory-prediction} we record the performance
guarantees (a direct consequence of \cite[Theorem 11]{flaspohler2021online}) of ODAFTRL, and then a corollary relating the bounds to the prediction errors.

\begin{theorem}[Theorem 11 of \cite{flaspohler2021online}]
Fix \(\alpha > 0\). Denote by $|X|=\mathrm{diam}(X)$  and define
\begin{align*}
b_{t}
&=
\mathrm{huber}\!\Bigl(\bigl\lVert h_{t}\;-\;\nabla Z_{t-m:t}\bigr\rVert,\,\lVert \nabla Z_{t}\rVert\Bigr),
\\
a_{t}
&=
\|\mathcal{X}\|\min\!\Bigl(\bigl\lVert h_{t}\;-\;\nabla Z_{t-m:t}\bigr\rVert,\,\lVert \nabla Z_{t}\rVert\Bigr), 
\end{align*}
where $\mathrm{huber}(x,y)
\;=\;$
\[
\tfrac{1}{2}\,x^{2}
\;-\;
\tfrac{1}{2}\,\bigl(\lvert x\rvert \;-\;\lvert y\rvert\bigr)_{+}^{2}
\;\;\le\;
\min\!\Bigl(\tfrac{1}{2}\,x^{2},\,|y|\,\lvert x\rvert\Bigr).
\]
Consider the delayed upper bound (DUB) sequence
\[
\mu_{t+1}
=
\frac{2}{\alpha}\,\max_{j\le t-m-1}a_{j-m+1:j}
+
\frac{1}{\alpha}\,\sqrt{\sum_{i=1}^{t-m}\!\Bigl(a_{i}^{2}\;+\;2\,\alpha\,b_{i}\Bigr)}.
\]
If \(r\) is nonnegative, then  the ODAFTRL iterates satisfy that 
\begin{align*}
\mathcal{R}_{T}(Z)
&\le\Bigl(\frac{r(u)}{\alpha}+1\Bigr)\\
&\times \biggl(
2\,\max_{t\in[T]}a_{t-m:t-1}
+
\sqrt{\sum_{t=1}^{T}\!\Bigl(a_{t}^{2}+2\,\alpha\,b_{t}\Bigr)}
\biggr).
\end{align*}

\end{theorem}
Based on this theorem we have the following corollary:
\begin{corollary}\label{corollary:Regret_linear}
   Using ODAFTRL on the Forward function $Z_t$ yields to the following regret bound:
    \begin{align*}
       \mathcal{R}_T(Z)& \leq  (\frac{r(u)}{\alpha} +1) \Big(m\|\mathcal{X}\|+\sqrt{\|\mathcal{X}\|^2+\alpha} \Big)\\
       &\times \sqrt{\sum_{t=m}^{T}||h_t-\nabla Z_{t-m:t}||^2}.
    \end{align*}
\end{corollary}
\begin{proof}[Proof of Corollary]
From the definitions, we have:
\begin{equation*}
a_t^2\leq \|\mathcal{X}\|^2 \cdot \left\| h_t - \nabla Z_{t-m:t} \right\|^2\!, 
\!\!\quad b_t
\leq \frac{1}{2} \cdot \left\| h_t - \nabla Z_{t-m:t}\right\|^2\!.
\end{equation*}

Hence, we can write:
\[
\sqrt{\! \sum_{t=1}^{T} \left( a_t^2 + 2\alpha b_t \right) }
\leq\!
\sqrt{ \left( \|\mathcal{X}\|^2 + \alpha \right) \sum_{t=1}^{T} \left\| h_t - \nabla Z_{t-m:t} \right\|^2 }.
\]
Similarly, 
\begin{align*}
   \max_{t\in \c T}a_{t-m:t-1} &\leq    \max_{t\in \c T} \sum_{j=t-m}^{t-1}||h_j - \nabla Z_{j-m:j}||\\
   &\leq m \|\mathcal{X}\|  \max_{t\in \c T}||h_t-\nabla Z_{t-m:t}|| \\
   &\leq m\|\mathcal{X}\|\sqrt{\sum_{t=1}^{T}||h_t-\nabla Z_{t-m:t}||^2}
\end{align*}
\end{proof}
Now we return back to prove theorem \ref{Theom:memory-prediction}.
\begin{proof}[Proof of Theorem]
 Using the inequality \( \|x + y\|^2 \leq 2\|x\|^2 + 2\|y\|^2 \), we split the error:
\begin{align*}
&\mathcal{E}_T(Z) 
\leq 2 \sum_{t=m}^{T}
\left\| \sum_{s=t-m}^{t} \sum_{i=t-s}^{m} \left( \nabla \tilde{f}_{s+i}^{i} - \nabla f_{s+i}^{i} \right) \right\|^2
\\&+ 2 \sum_{t=m}^{T}
\left\| \sum_{s=t-m}^{t} \sum_{i=t-s}^{m} \Phi'(V_{s-m-1+i}) \left( \nabla \tilde{g}_{s+i}^{+,i} - \nabla g_{s+i}^{+,i} \right) \right\|^2
\end{align*}
the second summation starts from $i=t-s$ instead of $0$ since not all the elements of $h_t$ will need to be predicted and some of them are known as mentioned before. 
Now, since \( \Phi' \) is non-decreasing and \( V_t \) is increasing, we have:
\[
\Phi'(V_{s-m-1+i}) \leq \Phi'(V_{T}) \quad \text{for all } s \leq t \leq T.
\]
Thus,
\[
\mathcal{E}_T(Z) \leq 2 \left( \mathcal{E}_T(f) + \big( \Phi'(V_{T}) \big)^2 \mathcal{E}_T\left(g^+\right) \right).
\]
Referring back to Lemma \ref{lem:memory-prediction-forward-decomposition}, we get:
\begin{align*}
    &\Phi(V_{T}) - \Phi(V_{m-1}) +\mathcal{R}_T^{mc}\\
    & \leq \mathcal{R}_T(Z) + G(m+1)\Phi'(V_{T})\\
      &\leq \left(\frac{r(u)}{\alpha} +1\right) \Big(m|X|+\sqrt{|X|^2+\alpha} \Big) \sqrt{\mathcal{E}_T(\mathcal{L})}\\
      &+ G(m+1)\Phi'(V_{T})\\
          &\leq  C \Big(\sqrt{\mathcal{E}_T(f)} + \Phi'(V_{T}) \sqrt{\mathcal{E}_T(g^+)}\Big)\\
    &+ G(m+1)\Phi'(V_{T})
    \end{align*}
And thus we have:
\begin{align*}
    \mathcal{R}_T^{mc}&\leq -\exp(\lambda V_{T}) +2+C\sqrt{\mathcal{E}_T(f)} \\
    &+ \lambda C \exp(\lambda V_{T}) \sqrt{\mathcal{E}_T(g^+)}+\lambda G(m+1) \exp(\lambda V_{T})\\
    &\leq 2+ C \sqrt{\mathcal{E}_T(f)}\\
    &+  \exp\left(\lambda V_{T}\right)\left(-1 +\lambda \left(C\sqrt{\mathcal{E}_T(g^+)} + G (m+1)\right)\right)\\
   &=\mathcal{O}(\sqrt{\mathcal{E}_T(f)}),
\end{align*}

and knowing that $\mathcal{R}_T^{mc}\geq -2FT$ we have: 
\begin{align*}
  \exp(\lambda V_{T})  & \leq \frac{ 2FT+2 +C \sqrt{\mathcal{E}}_f}{1-\lambda (C\sqrt{\mathcal{E}_T(g^+)} + G (m+1)) }\Rightarrow\\
V_{T} &\leq \frac{1}{\lambda} \log \Big( \frac{2FT+2 +C \sqrt{\mathcal{E}}_f}{1-\lambda (C\sqrt{\mathcal{E}_T(g^+)} + G (m+1))}\Big)\\
&= \mathcal{O}\Big((\sqrt{\mathcal{E}_T(g^+)}+m)\log T\Big)
\end{align*}
Since in this case $V_T\!=\!\mathcal{V}_T^{mc}$, we get the desired result. 
\end{proof}

\subsection{Optimistic $\texttt{COCO-M}$}\label{app:appen_optimi-COCO-M}

In $\texttt{COCO-M}$ the constraint is memory–less,
\(g_t(x_t)\).  The procedure mirrors the $\texttt{COCO-M}^2$ case
with two minor adjustments. We define the surrogate loss
\[
\mathcal L_t\bigl(x_{t-m}^{\,t}\bigr)
=\;
f_t\bigl(x_{t-m}^{\,t}\bigr)
+\Phi'\!\bigl(V_{t-1}\bigr)\,g_t^{+}(x_t),
\]
where $V_t=\mathcal{V}_T^c$, and the forward function is:
\begin{equation*}
Z_t(x_t)
=\sum_{i=0}^{m} f_{t+i}^{(i)}(x_t)
+\Phi'\!\bigl(V_{t-1}\bigr)\,g_t^{+}(x_t),
\end{equation*}
obtained by setting $g_{t}^{(i)}\equiv0$ for $i\ne0$. The gradient used at time \(t\) is
\begin{equation}\label{eq:forward:COCO-M}
\nabla \tilde Z_{t+1}
 =\sum_{i=0}^{m}\nabla f_{t+1+i}^{(i)}
 +\Phi'\!\bigl(V_t\bigr)\,\nabla g_t^{+},
\end{equation}
where \(V_t\) is already known when the update is computed.  And thus using
\(V_{t-1}\) in the surrogate (instead of \(V_{t-m}\)) does not require additional predictions. With this choice, we have:
\begin{align*}
  &  \Phi\bigl(V_{T}\bigr) - \Phi\bigl(V_{m-1}\bigr)
    +\mathcal{R}_T^{mc} \leq    \mathcal{R}_{T}(Z)
    + G \,\Phi'\bigl(V_{T}\bigr).
\end{align*}

Observe that in the forward loss we multiply the penalty term by the \emph{entire} constraint \(g_t(x_t)\), not by a part of it as in
$\texttt{COCO-M}^2$.
We may hence work with the usual benchmark:
\[
\mathcal X_T
=\Bigl\{
      x\in X\;:\;
      g_t(x)\le0
      \quad\text{for all }t\le T
  \Bigr\}.
\]

The algorithm update here is identical to the
$\texttt{COCO-M}^2$ algorithm with three minor changes.
(i)  loss and constraint are observed as in line 3 (not line 4);
(ii) the aggregate violation \(V_t\) is updated according to line 5;
and (iii) in line 8 the predictive term \(h_t\) is formed using the new  \(\nabla \tilde{Z}_t\) as in \eqref{eq:forward:COCO-M}.

Specializing Theorem~\ref{Theom:memory-prediction} to the present setting, it gives:
\begin{align}
    \mathcal{R}_T^{mc} &= \mathcal{O}\left( \sqrt{\mathcal{E}_T(f)} \right), \\
    \mathcal{V}_T^{mc} &= \mathcal{O}\left( \max\left\{ \sqrt{\mathcal{E}_T(g^+)},\, G \right\} \log T \right),
\end{align}
where 
\[
\mathcal{E}_T(g^{+})
\;=\;
\sum_{t=m}^{T}
\bigl\|
\nabla\tilde{g}_{t}
-\nabla g_{t}
\bigr\|^{2},
\]
and the considered benchmark set is $\mathcal{X}_T$. Since $g_t$ depends only on $x_t$, separability is irrelevant on the constraint side.

\begin{algorithm}[t]
\caption{Doubling Trick for Static Regret with Constraint Prediction Errors}
\label{alg:doubling_static}
\begin{algorithmic}[1]
\REQUIRE Real values $T_1,E_1,G(m+1)$; Optimistic meta-algorithm $\mathcal{O}(\lambda)$ for a given value of $\lambda$  
\STATE Initialize: $\mu_1 = \psi(T_1, E_1)$,\quad $\lambda_1 = \frac{1}{2(\mu_1 + c)}$,\quad $N = 1$  
\STATE Set: $\Delta_{(N)} = 0$,\quad $E_{(N)} = 0$,\quad $\mu_{(N)} = \psi(\Delta_{(N)}, E_{(N)})$  
\FOR{round $t = 1$ \TO $T$}
    \IF{$\mu_{(N)} > \mu_N$} 
        \STATE $N \gets N + 1$  
        \STATE $\mu_N \gets 2^{N-1} \mu_1$,\quad $\lambda_N \gets \frac{1}{2(\mu_N + c)}$  
        \STATE Reset: $\Delta_{(N)} \gets 0$,\quad $E_{(N)} \gets 0$
    \ENDIF
    \STATE Run one step of $\mathcal{O}(\lambda_N)$ and observe $f_t, g_t, x_t$
    \STATE $\Delta_{(N)} \gets \Delta_{(N)} + 1$
    \STATE $E_{(N)} \gets E_{(N)} + \varepsilon_t(g^+)$
    \STATE $\mu_{(N)} \gets \psi(\Delta_{(N)}, E_{(N)})$
\ENDFOR
\end{algorithmic}
\end{algorithm}

\subsection{Doubling Trick }\label{appen:doubling_trick}
To tune parameter $\lambda$ without prior knowledge of the cumulative prediction error $\mathcal{E}_T(g^+)$, we employ a doubling trick as suggested in \cite{lekeufack2024optimistic} and detailed in
Algorithm~\ref{alg:doubling_static}. The algorithm begins with an initial complexity estimate $\mu_1$, from which it sets $\lambda_1 = 1/(2(\mu_1 + c))$. At each round, the algorithm tracks an empirical estimate $\mu_{(N)} = \psi(\Delta_{(N)}, E_{(N)})$, where $\Delta_{(N)}$ is the number of rounds in the current epoch and $E_{(N)}$ is the cumulative prediction error observed so far. If the empirical complexity exceeds the current budget $\mu_N$, a new epoch is triggered with $\mu_{N+1} = 2\mu_N$, and the parameter is updated to $\lambda_{N+1} = 1/(2(\mu_{N+1} + c))$. This ensures that, during each epoch (except possibly the final round), the parameter $\lambda_N$ satisfies $\lambda_N \le \lambda_t^\star\!=\! 1/(2(\mu_{(N)}\!+\!c))$, where $\lambda_t^\star$ is the local oracle value based on complexity observed up to $t$. Although $\lambda_N$ may not match the global optimum $\lambda^\star$, it remains valid throughout the epoch, preserving the theoretical guarantees. The total number of restarts is at most $O(\log T)$, which contributes only an additive $O(\log T)$ overhead to regret and CCV bounds.

\end{document}